\documentclass{article}
\usepackage{amsmath,amssymb,amsthm,mathtools}
\usepackage{xspace}

\theoremstyle{plain}
\newtheorem{theorem}{Theorem}
\newtheorem{lemma}{Lemma}
\newtheorem{proposition}{Proposition}
\newtheorem{corollary}{Corollary}
\theoremstyle{definition}

\usepackage{latex_macros}
\renewcommand{\figdir}{.}

\providecommand{\real}{\mathbb{R}}
\providecommand{\Exp}{\mathbb{E}}

\providecommand{\mycomment}[1]{}

\usepackage{fullpage}
\usepackage{xcolor}
\IfFileExists{fontawesome.sty}{
\usepackage{fontawesome5}

}{

}

\usepackage[
colorlinks=true,
linkcolor=magenta,
citecolor=blue,
urlcolor=blue
]{hyperref}

\usepackage[capitalize,noabbrev,nameinlink]{cleveref}
\crefname{assumption}{Assumption}{Assumptions}
\Crefname{assumption}{Assumption}{Assumptions}

\usepackage{enumitem}

\newcommand{\Xvar}{\ensuremath{X}}
\newcommand{\usedim}{d}

\newcommand{\defn}{\coloneqq}

\newcommand{\lam}{\lambda}

\newcommand{\Denoise}{\mu}
\newcommand{\Info}{\mathsf{I}}

\newcommand{\Sinfo}{\mathsf{S}}
\newcommand{\Ginfo}{\mathsf{G}}
\newcommand{\Ghat}{\widehat{\Ginfo}}

\newcommand{\DenKL}{\Gamma_{\scaleto{\operatorname{prod}}{5pt}}}

\newcommand{\Law}{\mathcal{L}}

\newcommand{\KL}{\ensuremath{D_{\scaleto{\operatorname{KL}}{4pt}}}}

\newcommand{\Zvar}{Z}

\newcommand{\Zhat}{\ensuremath{\widehat{Z}}}

\newcommand{\Xhat}{\ensuremath{\widehat{X}}}

\newcommand{\ShannonInfo}{\ensuremath{\operatorname{Info}}}
\newcommand{\ShanInfo}{\ShannonInfo}

\newcommand{\Prob}{\ensuremath{\mathbb{P}}}

\newcommand{\Ent}{\ensuremath{\operatorname{Ent}}}

\newcommand{\Partition}{\ensuremath{\mathcal{P}}}

\renewcommand{\Xhat}{\widehat{X}}

\newcommand{\Alphabet}{\ensuremath{\mathcal{A}}}

\newcommand{\jind}{i}

\newcommand{\ind}{i}

\makeatletter
\let\MaskOriginalNewCommand\newcommand
\def\newcommand#1{%
\@ifundefined{\expandafter\@gobble\string#1}
{\MaskOriginalNewCommand{#1}}
{\renewcommand{#1}}}
\makeatother

\newcommand{\DenoiseHat}{\widehat{\Denoise}}
\newcommand{\DenHat}{\ensuremath{\DenoiseHat}}
\newcommand{\Info}{\ShanInfo}

\newcommand{\Dinfo}{\mathsf{D}}
\newcommand{\Ifun}{\ensuremath{C}}
\newcommand{\Bivq}{\ensuremath{\mathsf{Q}_{\scaleto{\igc}{4pt}}}}

\newcommand{\Law}{\mathcal{L}}

\newcommand{\KL}{\ensuremath{D_{\scaleto{\operatorname{KL}}{4pt}}}}

\newcommand{\rhat}{\widehat r}

\newcommand{\Zvar}{Z}
\newcommand{\Zsam}[1]{\ensuremath{Z^{(#1)}}}

\newcommand{\Zhat}{\ensuremath{\widehat{Z}}}

\newcommand{\ShannonInfo}{\ensuremath{\operatorname{Info}}}
\newcommand{\ShanInfo}{\ShannonInfo}

\providecommand{\Order}{\mathcal{O}}

\newcommand{\Prob}{\ensuremath{\mathbb{P}}}

\newcommand{\Ent}{\ensuremath{\operatorname{Ent}}}

\newcommand{\Partition}{\ensuremath{\mathcal{P}}}

\newcommand{\PartComp}{\mathsf{C}_{\scaleto{\operatorname{\igc}}{4pt}}}
\newcommand{\PartCompHat}{\widehat{\mathsf{C}}_{\scaleto{\operatorname{\igc}}{4pt}}}

\newcommand{\Xhat}{\widehat{X}}

\newcommand{\bind}{k}
\newcommand{\Btot}{K}

\newcommand{\block}{b}
\newcommand{\jind}{i}

\newcommand{\igclong}{interaction growth complexity\xspace}

\newcommand{\prd}{\textsf{PRD}\xspace}
\newcommand{\prdlong}{product-reference diffusion\xspace}

\newcommand{\dgc}{\textsf{DGC}\xspace}

\newcommand{\ugc}{\textsf{UGC}\xspace}
\newcommand{\igc}{\textsf{IGC}\xspace} \newcommand{\SR}{SR\xspace}

\newcommand{\Ber}{\textsf{Ber}\xspace}

\newcommand{\numobs}{\ensuremath{m}}

\newcommand{\qdens}{\mathsf{q}}

\newcommand{\rhohat}{\ensuremath{\widehat{\rho}}}
\newcommand{\PartHfine}{\mathsf{P}_{\scaleto{\mathrm{\igc}}{4pt}}}

\newcommand{\DenError}{\ensuremath{\mathcal{E}}}

\makeatletter \long\def\@makecaption#1#2{
\vskip 0.8ex \setbox\@tempboxa\hbox{\small {\bf #1:} #2} \parindent
1.5em \dimen0=\hsize 
\advance\dimen0 by -3em \ifdim \wd\@tempboxa >\dimen0 \hbox to \hsize{
  \parindent 0em \hfil
\parbox{\dimen0}{\def\baselinestretch{0.96}\small {\bf #1.} #2
}\hfil} \else \hbox to \hsize{\hfil \box\@tempboxa \hfil} \fi }
\makeatother

\newenvironment{researchquestion}
               {\begin{center}\begin{minipage}{0.98\textwidth}}
               {\end{minipage}\end{center}}
               \let\newcommand\MaskOriginalNewCommand

 \newcommand{\onevec}{\mathbf{1}}

\newcommand{\Uvar}{\ensuremath{U}}

 \newcommand{\rfinal}{\rtime_N}

\newcommand{\newrfinal}{R}

\newcommand{\newrinit}{{\rtime_0}}

\newcommand{\PlainKer}{\ensuremath{\mathbb{K}}}

\newcommand{\Kexact}[2]{\PlainKer_{#1, #2}}
\newcommand{\Khat}[2]{\widehat{\PlainKer}_{#1, #2}}
\newcommand{\Ktilde}[2]{\widetilde{\PlainKer}_{#1, #2}}

\newcommand{\rtime}{r}
\newcommand{\odds}{\ensuremath{\psi}}

\newcommand{\logit}{\ensuremath{\varphi}}
 \newcommand{\logodds}{\logit}

\newcommand{\revinv}{\logit^{-1}}

\newcommand{\DTC}{\ensuremath{\mathsf{DTC}}}
\newcommand{\TC}{\ensuremath{\mathsf{TC}}}

\newcommand{\Elld}{\ell_\usedim}

\newcommand{\SB}{\PartComp} \newcommand{\FP}{\PartHfine}

\newcommand{\BOUNDARY}{ \TC(\Xvar_{\rtime_0}) + \TC(\Zvar \mid
  X_\tfinal)}

\newcommand{\BOUNDARYSPEC}{ \TC(\Xvar_{\rtime_0}) + \TC(\Zvar \mid
  X_{1 - \rtime_0})}

\newcommand{\Ker}{\ensuremath{\mathbb{K}}}
\newcommand{\IntStar}{\ensuremath{\Int_{\scaleto{\operatorname{full}}{4pt}}}}
 \newcommand{\Int}{\mathcal{I}}

\newcommand{\Ratio}{\operatorname{Ratio}}
\newcommand{\ProbZ}{\ensuremath{\Prob_\Zvar}}

\newcommand{\order}{\Order}

\newcommand{\newrevinv}{r}

\newcommand{\gfun}{\ensuremath{\mathsf{g}}}

 \newcommand{\Tmap}{\ensuremath{\mathbb{T}}}

\newcommand{\rbold}{\mathbf{r}} \newcommand{\tbold}{\mathbf{t}}

 \newcommand{\evec}{\mathbf{e}}

\newcommand{\blam}{\boldsymbol{\lambda}}

\renewcommand{\time}{r} \newcommand{\tfinal}{R}

\newcommand{\rstar}{r_0}

\newcommand{\KerTil}{\ensuremath{\mathbb{\widetilde{K}}}}
\newcommand{\Refdist}{\mathbb{V}} \newcommand{\refdist}{\nu}

\newcommand{\Tfun}{T}
\newcommand{\Dfun}{D}

\newcommand{\Pmap}{\ensuremath{\mathbb{T}}}
\newcommand{\Breg}{\mathsf{B}}
\newcommand{\ehack}{e(\alpha, \beta)}

\newcommand{\Den}{\mu}
\newcommand{\DenLoo}{\Den^{\scaleto{\operatorname{LOO}}{2pt}}}

\newcommand{\BridgeAB}{\ensuremath{\mathbb{B}^{a \to b}}}

\begin{document}

\begin{center}
  {\Large\bfseries The information geometry of product-reference
    discrete diffusion: \\ Interaction growth complexity and 
    optimal scheduling}
\vspace*{0.3in}

\begin{tabular}{c}
Martin J. Wainwright \\ \texttt{mjwain@mit.edu}
\end{tabular}

\vspace*{0.2in}
\begin{tabular}{c}
Lab for Information and Decision Systems \\ Statistics and Data
Science Center \\ EECS and Mathematics, \\ Massachusetts Institute of
Technology
\end{tabular}

\vspace*{0.25in} \today
\vspace*{0.25in}

\begin{abstract}
We study a class of product-reference diffusion algorithms for
sampling from a discrete distribution.  We show that their sampling
performance can be characterized using a path-based measure of data
geometry that we call the \emph{interaction growth complexity}
({\textsf{IGC}\xspace}).  We show that a bivariate
{\textsf{IGC}\xspace} kernel gives an exact representation of both the
KL discretization error and a simple one-step upper bound. The simpler
univariate {\textsf{IGC}\xspace} density can be used to study the
effect of stepsize choices on the iteration complexity required to
obtain $\varepsilon$-accurate samples in KL divergence.  Samplers that
traverse the path with equi-spaced steps in
log-squared-reliability-odds have performance that depends on the
aggregate {\textsf{IGC}\xspace} mass, whereas refined choices of
stepsizes have a lower complexity depending on a square-root
functional.  In the fine-grid limit, both of these characterizations
become sharp.  We also allow general product reference distributions
and show that the reference law can substantially reshape the
{\textsf{IGC}\xspace} profile and the resulting sampling complexity;
in particular, references far from both the uniform and the data
marginals can yield dimension-dependent improvements.  Finally, the
aggregate {\textsf{IGC}\xspace} mass admits bounds in terms of total
correlation and dual total correlation, thereby connecting the
pathwise geometry to classical measures of multivariate dependence.
\end{abstract}
\end{center}



\section{Introduction}
Sampling from complex high-dimensional distributions is a basic
computational problem arising throughout statistics, machine learning,
and the sciences.  It provides the foundation for Monte Carlo
methods~\cite{RobCase04,RubKroe08}, enables uncertainty quantification
in Bayesian inference~\cite{GelEtAl13,BroEtAl11}, and is a central
component of modern generative modeling~\cite{CheEtAl24}.
Diffusion-based methods approach this problem
by constructing a family of progressively corrupted distributions,
together with a reverse procedure that removes the corruption and
ultimately produces a sample from the target
law~\cite{SohEtAl15,SonErmo19,HoEtAl20,SonEtAl21,CheEtAl23c,
  CheEtAl23a,BenEtAl24}.  Much of the early theory was developed for
continuous state spaces, where Gaussian perturbations provide the
canonical forward process and the reverse dynamics are governed by
Gaussian denoising~\cite{HoEtAl20,SonEtAl21}.  For distributions with
discrete support, however, Gaussian perturbations are no longer
natural.  This fact has motivated a growing literature on discrete
diffusion models, in which the forward process is instead built from
Markovian corruption kernels on the underlying alphabet, and sampling
is carried out by approximately reversing these transitions (e.g., see
the papers~\cite{HooEtAl21,AusEtAl21,CamEtAl22,LouEtAl24,ShiEtAl24b,
  SahooEtAl2024MDLM} and references therein).  Understanding the
computational and statistical properties of such discrete diffusion
schemes is the focus of this paper.

\paragraph{Concurrent work:}  
Concurrent and independent work of Dmitriev, Huang and
Wei~\cite{DmiHuaWei26b}, posted on August 24, 2026 as this manuscript
was being finalized, studies related questions for discrete diffusion
sampling within a CTMC formalism.  We plan to coordinate with the
authors in a future revision to clarify the connections between the
two approaches.


\subsection{Our contributions}

Let us provide a high-level overview of the contributions of this
paper.  As with our previous papers on Gaussian
diffusion~\cite{Wai_DGC} and masking diffusion~\cite{Wai_UGC}, this
paper is motivated by two broad questions.  Tailoring their form to
the current setting, we ask:
\begin{researchquestion}
{\bf{Q1:}} Can the performance of discrete diffusion sampling be
explained and quantified, in some generality, by a measure tied to
information and data geometry?  \par\smallskip
{\bf{Q2:}} Is it possible to exploit such geometric measures to
design, optimize and certify sampling schemes?
\end{researchquestion}

In our previous papers~\cite{Wai_DGC,Wai_UGC}, we showed that some
fruitful answers to both questions can be given by specifying a a
suitable form of complexity for the path from ``noise'' to a fresh
sample.  For Gaussian diffusion and masked diffusion, we referred to
these notions as denoising growth complexity (\dgc), and unmasking
growth complexity (\ugc), respectively.  Both complexity measures are
defined in information-theoretic terms, via the second derivative of a
mutual information along the path.  Moreover, they lead to a geometric
understanding of the performance of sampling algorithms ({\bf{Q1}}),
while also providing a platform for optimizing and certifying their
performance ({\bf{Q2}}).

The main contribution of this paper is to extend this type of theory
via the notion of \emph{interaction growth complexity}, or \igc for
short.  This complexity measure applies to a standard class of
discrete diffusion samplers (e.g.,~\cite{HooEtAl21,AusEtAl21}) that
use product approximations.  They traverse a noising path terminating
at an arbitrary product-reference distribution, using a sequence of
product approximations along the path, so that we refer to them as
\emph{product-reference diffusion} samplers, or \prd for short. The
uniform distribution and an absorbing-state reference distribution
($\delta$-function at a single state) are two standard choices of the reference
distribution, but our treatment allows for an arbitrary choice.  Our
analysis shows that the log-odds of the squared reliability (\SR) is a
natural parameter, and we define an \igc complexity that is based on
the off-diagonal components of the Hessian of the mutual information
along the path (see equation~\eqref{EqnDefnGfun}).  We show that a
simple \prd algorithm, one which traverses the path with steps of
constant spacing in log-\SR-odds, has iteration complexity governed by
the aggregate \igc mass.  This sampler adapts to low-dimensional
geometric structure, since the aggregate \igc mass can be upper
bounded by the classical $\TC/\DTC$ measures of multivariate
dependence~\cite{Wat60,Han78} that are sensitive to it.  The stepsizes
for this simple scheme are oblivious, requiring no knowledge of this
structure nor clean samples for estimation.  However, if in addition we
are given access to clean samples from the target distribution, then
we show how the structure of the \igc path can be estimated via a
natural Bregman divergence that lies at the core of our theory.
Tailoring stepsize choices to this geometric structure leads to
samplers with improved and certified performance.

The \igc theory, while sharing many features with its \dgc and \ugc
precedents, differs in an important and interesting way.  For both
Gaussian diffusion and masking diffusion, the complexity structure is
essentially \emph{univariate} in nature.  Both the exact KL
discretization error and one-step error bounds can be expressed as
functions of a univariate (\dgc or \ugc) function.  In this paper, we
put forth a univariate \igc density that provides similar intuition
and guarantees, but note that the true complexity structure is
\emph{bivariate} in nature.  This bivariate nature can be seen from
the definition~\eqref{EqnDefnGfun} of the \igc function, which
involves mixed partial derivatives of a mutual information.
Section~\ref{SecBivQdens} is devoted to an in-depth exploration of
this bivariate structure, which connects to the $\TC$ and $\DTC$
measures of multivariate dependence.

\subsection{Related work}
\label{SecRelated}

Here we provide an overview of some broad strands of related work,
with particular focus on those most closely related to our analysis of
discrete diffusion.  Sohl-Dickstein et al.~\cite{SohEtAl15} described
diffusion-style generative modeling for both continuous and discrete
state spaces.  Later work focused more specifically on categorical
data, and developed multinomial and structured forms of discrete
diffusion processes~\cite{HooEtAl21,AusEtAl21}.  Campbell et
al.~\cite{CamEtAl22} placed discrete diffusion in a continuous-time
Markov chain (CTMC) framework; the work of Benton et
al.~\cite{BenEtAl24} provided a broader perspective on denoising
Markov models.  Analogues of score matching and score-based modeling
were developed by Meng et al.~\cite{MenEtAl22} and Lou et
al.~\cite{LouEtAl24}; the latter introduced the score-entropy loss
that plays an important role in training.  While not the focus of this
paper, we note that discrete diffusion based on masking has received
substantial
attention~\cite{ShiEtAl24b,SahooEtAl2024MDLM,OuEtAl25,DmiEtAl26,CheConLi26}.

In the current paper, we focus on a standard class of
product-reference diffusion samplers (see
equations~\eqref{EqnGeneralRef} and~\eqref{EqnGeneralRefTransition}
for the noising path), and we allow the product-reference distribution
to have arbitrary coordinate-dependent components.  This class falls
within the broader transition-matrix design space of discrete
diffusion elucidated in past work~\cite{AusEtAl21,CamEtAl22}.  The
product-reference updates that we analyze involve replacing an
intractable joint reverse transition by simultaneous coordinatewise
updates.  This type of approximation is shared by continuous-time
Markov chain (CTMC) approaches, beginning with classical
$\tau$-leaping, with later adaptations to discrete
diffusion sampling by Campbell et al.~\cite{CamEtAl22}.  Lou et
al.~\cite{LouEtAl24} introduced Tweedie $\tau$-leaping and showed,
under exact scores, that its tokenwise independent transition is
KL-optimal among $\tau$-leaping strategies.  While we do not adopt the
CTMC formalism, the product kernel~\eqref{EqnProductKernel} in this
paper has the same KL optimality property: for each input state, it is
the product of the exact one-coordinate marginals of the finite
reverse bridge.  Consequently, its one-step KL defect is a
(conditional) total correlation, and we make essential use of this
structure in our analysis.

In recent work, Gourevitch et al.~\cite{GouEtAl26} clarify important
choices in how to parameterize the product kernel for uniform discrete
diffusion, and in particular the distinct roles played by the ordinary
one-site posterior and its leave-one-out counterpart.  At the
population level, the appropriate use of either quantity recovers the
same oracle one-site reverse marginals, and hence the same product
reverse kernel.  Our theory analyzes the KL discretization error of
this oracle product kernel, independently of the particular
parameterization used to realize it.  See the discussions preceding
and following~\Cref{ThmMaster} for more details.

In terms of convergence guarantees for discrete diffusion, there is
now a substantial non-asymptotic theory, including stochastic-integral
and Girsanov analyses~\cite{RenEtAl25,ZhaCheGu25}, the DMPM framework
of Pham et al.~\cite{PhaEtAl25}, and analyses of $\tau$-leaping,
Euler, and Tweedie-$\tau$-leaping under uniform or absorbing
dynamics~\cite{LiangEtAl25,LiangEtAl25b,ConDurPha25}.  Most related to
our development is recent work that has made important progress in
giving general geometric characterizations of sampling performance,
showing how they adapt to low-dimensional latent structure.  Focusing
on masking diffusion, two recent papers~\cite{CheConLi26,DmiEtAl26}
connected sampling performance to the total correlation ($\TC$) and
dual total correlation ($\DTC$) measures of multivariate
dependence~\cite{Wat60,Han78}.  Dmitriev et al.~\cite{DmiEtAl26}
provided a sharp analysis of a class of modified $\tau$-leaping CTMC
samplers. For masking diffusion, they showed that its performance is
governed by a refined measure that is upper bounded by the minimum of
the $\TC/\DTC$ measures.  On the other hand, for a non-masking
sampler, they proved an upper bound that scales with
dimension. Intriguingly, they also complemented this upper bound with
an algorithmic lower bound, showing that the particular CTMC scheme
analyzed in their paper fails to adapt to low-dimensional structure.
Thus, they raised the natural open question as to whether an
alternative sampler for discrete diffusion might adapt to
low-dimensional structure.  In this paper, we show that \prd samplers
adapt to $\TC/\DTC$ structure, thereby providing an
affirmative answer to this question.

Our set-up and theory allow for a general choice of reference
distribution, and this choice has also been studied empirically.
Austin et al.~\cite{AusEtAl21} compare several forward corruption
kernels for discrete diffusion, finding substantial performance
differences between \emph{uniform} and \emph{absorbing-state} noise.
DiGress~\cite{VigEtAl23} instead uses a terminal reference that is
\emph{matched} to the one-coordinate marginals, and reports
improvements over uniform corruption in graph generation.  Other work
compares uniform, marginal-matched, and absorbing transitions for
graph diffusion~\cite{LaaEtAl25}, as well as uniform and absorbing
mechanisms across other discrete-data domains and sampling
regimes~\cite{SchEtAl25}. Collectively, there is considerable
empirical evidence that the choice of terminal reference can have
significant effects on sampling performance.  Our theoretical
framework provides a crisp way of studying and characterizing this
choice, as discussed in~\Cref{SecReference}.  Notably, we show that
there is no simple ordering relation among these choices. Our method
is constructive in nature: we exhibit an ensemble of distributions in
which the gap between the marginal-matched reference and uniform/absorbing-state
choices grows as $\sqrt{\usedim}$ in dimension, and another ensemble
in which the absorbing-state choice exhibits a $\log \usedim$
deficiency relative to the uniform choice.  This analysis gives
insight into what geometric structure of the underlying $\ProbZ$
controls the effect of the terminal reference, and suggests directions
for future work.

Finally, at a higher level, our analysis has important structural
connections with our previous work~\cite{Wai_DGC,Wai_UGC} on Gaussian
and masking diffusion, respectively. At the core of all three papers
is the identification of a complexity measure (\dgc, \ugc, or \igc)
that simultaneously yields an exact representation of the one-step KL
error and leads to geometric bounds on it. The same complexity measure
controls the KL error associated with grids that are equi-spaced in a
privileged time parameter, and dictates how the step sizes should be
optimized. Much of the analysis across these three settings therefore
follows a common template, revealing structural connections between
diffusion samplers whose technical details are otherwise substantially
different.


\section{Overview: Interaction growth complexity and sampling}
\label{SecOverview}

The main ideas of this paper can be conveyed at a relatively high
level, and this section is devoted to this end.  We begin by
describing a standard class of product-reference diffusion (\prd)
samplers. Given a fixed reference distribution $\refdist$ that
factorizes across coordinates, any \prd sampler generates samples by
tracing the path from the reference $\refdist$ to the target
distribution $\ProbZ$, using product-based approximations to the exact
transition kernels at each step. We define the notion of interaction
growth complexity, which is a way of tracking the evolution of
information along this path.  Its structure both reveals geometric
properties of the target distribution, and controls the complexity of
sampling algorithms based on kernel-product approximations.  The
purpose of this section is to give a conceptual overview; subsequent
sections provide the technical content that underlies this view.

\subsection{General-reference diffusion}

Given a discrete alphabet $\Alphabet$ and a dimension $\usedim$, our
goal is to draw a sample $Z \in \Alphabet^\usedim$ from its unknown
probability distribution $\ProbZ$.  In this paper, we study a simple
class of discrete diffusion samplers, which move progressively from
the original distribution $\ProbZ$ to a reference distribution.  In
particular, throughout the paper, we use $\refdist = (\refdist_1,
\ldots, \refdist_\usedim)$ to denote a collection of
$\usedim$-marginals, possibly different from coordinate to coordinate,
that define the product measure $\otimes_{i=1}^\usedim \refdist_i$ on
the space $\Alphabet^\usedim$.  In the simplest case, known as uniform
discrete diffusion, we simply set each $\refdist_i$ to be uniform over
the alphabet $\Alphabet$.

For coordinate $i$ and squared reliability\footnote{To explain the
terminology, for any function $f: \Alphabet \rightarrow \real$ that is
mean zero under the reference distribution $\refdist_i$, we have
$\Exs[ f(X_{i,r}) \mid Z_i] = \sqrt{r} f(Z_i)$ for each $i = 1,
\ldots, \usedim$, so that any given coordinate is attenuated by the
reliability $\sqrt{r}$.  Our choice of squared reliability for
parameterizing the path, while unorthodox, turns out to be revealing
since it is the natural parameter for data-processing inequalities for
this channel (see~\Cref{LemDPI}) and leads to exponential sandwiches of
information-theoretic objects (see~\Cref{LemBivSandwich}).  }  $r \in
[0,1]$, we define the $\refdist$-reference transition via
\begin{subequations}
\begin{align}
  \label{EqnGeneralRef}
  \Refdist_{i,r}(x_i \mid z_i) & \defn \sqrt{r} \Ind[x_i = z_i] + (1 -
  \sqrt{r}) \refdist_i(x_i).
\end{align}
Note that the endpoint $r = 0$ corresponds to a Markov transition that
produces the reference distribution $\refdist$.  Moreover, it is easy to
see that we have the semigroup identity $\Refdist_r \Refdist_s =
\Refdist_{rs}$ for any pair $(r,s) \in (0,1)$.

We now use this semigroup to define a stochastic process $\{\Xvar_r
\mid r \in [0, 1]\}$.  It has endpoint $\Xvar_1 \sim \Zvar$, and moves
downwards in squared-reliability (\SR) time to the reference endpoint
$\Xvar_0 \sim \otimes_{i=1}^\usedim \refdist_i$.  For every $0 < a < b
\leq 1$, conditional on $\Xvar_b$, the coordinates of $\Xvar_a$ are
independent and satisfy
\begin{align}
  \label{EqnGeneralRefTransition}
  \underbrace{\Prob(\Xvar_{a, \jind} = x \mid \Xvar_{b, \jind} =
    y)}_{\equiv \Refdist_{a/b}(x \mid y)} = \sqrt{\frac{a}{b}}
  \mathbf{1}\{x = y\} + \big (1 - \sqrt{\frac{a}{b}} \big)
  \refdist_{\jind}(x).
\end{align}
Stated more simply in words, the transition from time $b \to a$ is
obtained by sampling each coordinate $\Xvar_{a, \ind}$ independently,
where we retain the value $\Xvar_{b, i}$ with probability $\sqrt{a/b}$
and replace it with a fresh sample from the reference distribution
$\refdist_i$ with probability $1 - \sqrt{a/b}$.  Note that the
semigroup identity ensures that $\Refdist_{a/b} \Refdist_{b/c} =
\Refdist_{a/c}$ so that a transition from $c$ to $b$ followed by that
from $b$ to $a$ equals the direct transition from $c$ to $a$.  This
consistency makes $\{\Xvar_r \mid r \in [0, 1]\}$ a genuine
continuous-parameter Markov process rather than a collection of
separately defined marginals.

So as to simplify our analysis in the sequel, we have chosen to define
the noising process in reverse time, moving from $\rtime = 1$ to
$\rtime = 0$.  Accordingly, sampling is carried out by approximating
the \emph{forward evolution} from a sample $\Xvar_0$ drawn from the
product-reference distribution to the final sample $\Xvar_1
\equiv \Zvar$.  A straightforward calculation shows that, for any pair
\mbox{$0 < a < b < 1$,} the probability transition from $X_a$ to $X_b$
is defined by the kernel
\begin{align}
\label{EqnKerExact}  
  \Ker_{a, b}(y \mid x) & \defn \Prob(\Xvar_b = y \mid \Xvar_a = x) =
  \frac{p_b(y)}{p_a(x)} \prod_{i=1}^\usedim \Refdist_{i, a/b}(x_i \mid
  y_i),
\end{align}
\end{subequations}
where $p_r: (\Alphabet)^\usedim \rightarrow [0,1]$ is the marginal
probability mass function of $X_r$.

In general, the exact kernel~\eqref{EqnKerExact} is computationally
intractable, since the state space has cardinality that grows
exponentially in $\usedim$.  In this paper, we analyze a standard
approximation in which we instead sample from a family of product
kernels $\KerTil_{a, b}$, obtained, for each fixed $x$, by taking the
product of the one-coordinate marginals of the exact conditional
kernel $\Ker_{a, b}(\cdot \mid x)$; for example, see Lou et
al.~\cite{LouEtAl24}.  We refer to this sampler as the \emph{\prdlong}
sampler, or \prd for short, since it is based on a product
approximation to a reference-based diffusion process.  By changing the
reference distribution $\refdist$, we obtain different types of
algorithms; for instance, choosing a uniform $\refdist$ gives the
standard uniform diffusion sampler~\cite{HooEtAl21,AusEtAl21};
choosing $\refdist$ to be a $\delta$-function on a particular state
yields an absorbing-state sampler~\cite{AusEtAl21}; and matching
$\refdist$ to the marginal distributions of $\ProbZ$ yields a
marginal-matched sampler~\cite{VigEtAl23}.  Our theory applies to all
of these choices.  See~\Cref{SecAlg} for a precise description of the
class of algorithms that we analyze, and~\Cref{SecReference} for
some insight into the effects of different references on
sampler performance.


\subsection{Introducing the interaction growth complexity (\igc)}
\label{SecIGC}

We now introduce the complexity measure that plays the central role in
this paper.  It involves the second partial derivatives of the mutual
information between the latent variable $\Zvar \in \Alphabet^\usedim$,
and the vector \mbox{$\Xvar_\tbold = (X_{1, t_1}, \ldots, X_{\usedim,
    t_\usedim}) \in \Alphabet^\usedim$} along an inhomogeneous version
$\tbold \defn (t_1, \ldots, t_\usedim)$ of the squared-reliability
path.  More precisely, we define the function $\gfun:[0,1] \rightarrow
\real$ via
\begin{subequations}
\begin{align}
  \label{EqnDefnGfun}
\gfun(r) & \defn - \sum_{i \neq j} \left. \frac{\partial^2}{\partial
  t_i \partial t_j} \Info(Z; X_{\tbold}) \right|_{\tbold = r \onevec},
\end{align}
where $\Info(Z; X_{\tbold})$ denotes the mutual information between
$Z$ and $X_{\tbold}$.  Using this function, the \igclong assigns a
positive mass to any interval $[a,b] \subset (0,1)$ via the weighted
integral
\begin{align}
\label{EqnDefnIGC}
\Ginfo(a, b) & \defn \int_a^b \rtime (1 - \rtime) \gfun(\rtime)
\, d\rtime.
\end{align}
We refer to this weighted measure as the \igc mass function.
Equivalently, so as to remove the explicit weighting, we can define
the log-squared-reliability density $\qdens$ in the coordinates
$\logodds(r) = \log(r/(1-r))$.  It has the explicit form
\begin{align}
\label{EqnDefnQdens}
\qdens(\lam) & \defn r^2(\lam) \big (1 - r(\lam) \big)^2 \; \gfun
\big(r(\lam) \big) \qquad \mbox{where $r(\lam) = \revinv(\lam) \defn
  \frac{e^{\lam}}{1 + e^\lam}$.}
\end{align}
\end{subequations}
Note that we have $\Ginfo(a, b) = \int_{\logodds(a)}^{\logodds(b)}
\qdens(\eta) \, d\eta$ by construction.

In terms of the shorthand \SR for squared-reliability, the log-\SR
density $\qdens$ is the central object of this paper.  In the limit of
fine stepsizes, it determines the exact KL discretization error up to
a factor of two. It gives explicit upper bounds on practical
algorithms, both those based on stepsizes with constant spacing in
log-\SR-odds, and those with stepsizes optimized to minimize error.
Let us provide a high-level overview of these conclusions here, along
with pointers to the relevant technical sections.

To streamline presentation, we focus on squared-reliability intervals
of the form $[\rstar, 1 - \rstar]$, which transform to the
centered\footnote{The choice of $\Elld$ only affects the
initialization and termination error in a minor way: as we show in the
sequel, the choice $\Elld = \Theta(\log(\usedim |\Alphabet|)$ is
sufficient to drive these errors to zero at rate
$\operatorname{poly}(\frac{1}{\usedim |\Alphabet|})$.}  intervals
$[-\Elld, \Elld]$ under the log-\SR-odds transform $\Elld = \log\left(
\frac{1 - \rstar}{\rstar} \right)$.  With this set-up, we can
summarize our main guarantees as follows:
\begin{subequations}
\begin{itemize}[leftmargin=1em, topsep=2pt, itemsep=2pt,
  parsep=0pt, partopsep=0pt]
\item The simplest form of a \prd sampler traverses the interval
  $[-\Elld, \Elld]$ using steps with constant spacing in log-\SR-odds.
  As a simple corollary of~\Cref{ThmMaster}, we prove that this simple
  scheme has iteration complexity that scales as
  \begin{align}
\label{EqnDefnSingBlock}    
    \SB & \defn 2 \Elld \int_{-\Elld}^{\Elld} \qdens(\lam) \, d\lam.
  \end{align}
Moreover, as the grid size converges to zero, this iteration
complexity estimate is sharp up to a constant prefactor, in the sense
that the \emph{exact KL discretization error} exhibits the same
scaling.  See~\Cref{ThmKL} for this claim.  Moreover, we show how to
bound $\SB$ in terms of the classical complexity measures~\cite{Wat60,Han78}
known as total correlation ($\TC$) and dual total correlation
($\DTC$).  This shows how even this simple single-block sampler can
adapt to the intrinsic geometry of $\ProbZ$.

\item As we show below in~\Cref{SecGeometry}, for ``interesting''
  random vectors $\Zvar \sim \Prob_Z$, the density $\qdens$ has a rich
  geometric structure.  This suggests that better samplers can be
  obtained by tailoring the stepsizes to the structure of $\qdens$,
  with finer steps in portions where $\qdens$ is rapidly varying.  We
  exhibit a sequence of practical samplers, using progressive
  refinement with constant spacing blocks, whose iteration complexity
  converges to the \emph{fine partition limit}
    \begin{align}
      \label{EqnDefnFinePart}
      \FP & \defn \left (\int_{-\Elld}^{\Elld} \sqrt{\qdens(\lam)}
      \, d\lam\right)^2,
    \end{align}
\end{itemize}
\end{subequations}

Taken collectively, these results lead to a simple characterization of
the potential benefits of optimizing stepsizes in any
product-reference sampler.  In particular, the gains are characterized
by the ratio
\begin{align}
\label{EqnDefnRatio}
  \Ratio(\ProbZ) & \defn \frac{\SB}{\FP} \; = \; \frac{2 \Elld
    \int_{-\Elld}^{\Elld} \qdens(\lam) \, d\lam}{ \left
    (\int_{-\Elld}^{\Elld} \sqrt{\qdens(\lam)} \, d\lam\right)^2} \;
  \stackrel{(i)}{\geq} \; 1,
\end{align}
where the lower bound (i) follows from the Cauchy--Schwarz inequality,
with equality holding if and only if $\qdens$ is constant.  When $\qdens$
is unevenly distributed, then the ratio becomes much larger than one.
As we show in the next section, such uneven distributions arise in
many settings, and often reflect interesting geometric structure of
the distribution.

Finally, our explicit characterization of \prd sampling complexity also leads
to insights into the broader design space of sampling algorithms:
\begin{itemize}[leftmargin=1em, topsep=2pt, itemsep=2pt,
    parsep=0pt, partopsep=0pt]
\item Our streamlined notation hides the fact that the \igc density
  $\qdens \equiv \qdens_\refdist$ depends on the choice of reference
  distribution $\refdist$, since it is embedded within the
  ``noising'' process~\eqref{EqnGeneralRef} that defines the \igc
  complexity.  In practice, the choice of $\refdist$ is an important
  design parameter, with empirical work showing that it has
  substantial effects on sampler performance. For each choice of
  reference $\refdist$, we have both a coarse
  complexity~\eqref{EqnDefnSingBlock} and a fine-partition
  complexity~\eqref{EqnDefnFinePart}.  By studying these complexities,
  we can obtain insights into the differences among samplers using
  uniform, absorbing-state, and marginal-matched references.
  See~\Cref{SecRelated} for some representative results.
\end{itemize}
  
\subsection{The data geometry of the \igc density}
\label{SecGeometry}

Since the \igc density controls the complexity of product-reference
diffusion, it is natural to wonder about what it reveals about the
underlying geometry.  In this section, we describe three simple
ensembles, each of which is chosen to illustrate a different
qualitative feature of the \igc density.

\subsubsection{Noisy repeated bit}
\label{SecNoisyRepeated}

We begin with a very simple example, also studied in our past
work~\cite{Wai_UGC} on masking algorithms.  For any $p \in [0,1]$, we
use $V \sim \Ber(p)$ to denote a Bernoulli random variable with
$\Prob(V = 1) = p$ and $\Prob(V = 0) = 1 - p$.

Consider a binary random vector $\Zvar = (\Zvar_1, \ldots,
\Zvar_\usedim) \in \{0, 1 \}^\usedim$ constructed in the following
way.  We first draw a latent variable $\Uvar \sim \Ber(1/2)$, and
then, with $\oplus$ denoting addition modulo two, we set
\begin{align}
\label{EqnNoisyBit}
\Zvar_\jind & = \Uvar \oplus W_\jind \qquad \mbox{for each $\jind = 1,
  \ldots, \usedim$,}
\end{align}
where each $W_\jind$ is an independent $\Ber(\eta)$-variable for some
$\eta \in [0, 1/2]$.  By construction, we have $\Prob(\Zvar_\jind =
\Uvar) = 1 - \eta$, hence our use of the term ``noisy repeated bit''.
For any $\eta \in [0, 1/2]$, the coordinates of $Z$ are conditionally
independent given $U$, but for any $\eta \in [0, 1/2)$, they are not
  independent, due to the shared structure.

Using the uniform reference measure to define the path, we computed
the $\igc$ density $\qdens$ for this ensemble at dimension $\usedim =
128$ and varying choices of the cross-over parameter $\eta \in (0,
1/2]$.
\begin{figure}[htb!]
\begin{center}
\begin{tabular}{@{}ccc@{}}
  \widgraph{0.3\textwidth}{\figdir/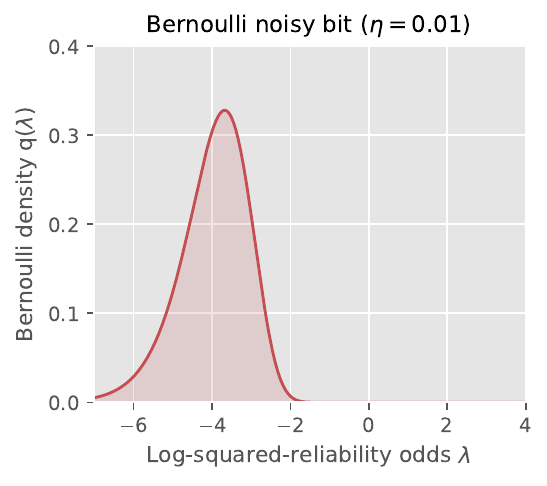}
  &
  \widgraph{0.3\textwidth}{\figdir/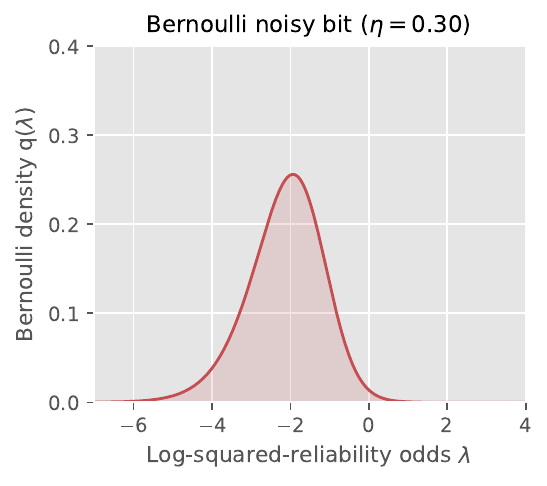}
  &
  \widgraph{0.3\textwidth}{\figdir/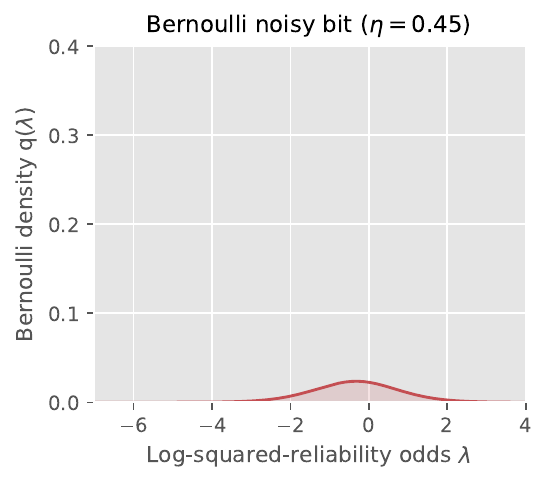}
  \\
  (a) & (b) & (c)
\end{tabular}
\caption{Plots of the \igc density $\qdens$ for the noisy repeated bit
  model in equation~\eqref{EqnNoisyBit}, for three different settings
  of the noise level $\eta \in \{0.01, 0.30, 0.45 \}$.  The \igc
  density captures the decaying dependence on the latent bit $U$ as
  $\eta$ evolves towards the endpoint $1/2$.}
\label{FigNoisyBit}
\end{center}
\end{figure}
As can be seen in~\Cref{FigNoisyBit}, when the cross-over probability
$\eta$ is small, the \igc density has a pronounced peak, reflecting
the strong dependence induced by the shared latent bit $\Uvar$. As
$\eta$ increases towards $1/2$, this peak becomes progressively
smaller, corresponding to the weakening of the common latent
structure. In the limiting case $\eta = 1/2$, each coordinate is an
independent $\Ber(1/2)$ variable, so that $\Prob_Z$ is itself a
product distribution and the \igc density vanishes identically.  As we
show in~\Cref{CorUpperTC}, the integrated \igc mass is upper bounded
by the total correlation
\begin{align*}
\TC(Z) & \defn \big \{ \sum_{i=1}^d \Ent(Z_i) \big \} - \Ent(Z),
\end{align*}
which approaches zero as $\ProbZ$ approaches a product distribution.
Overall, the evolution of $\qdens$ across the three panels illustrates
the adaptivity of the \igc density to the dependence structure of the
target.


\subsubsection{Curie--Weiss model}

Next we turn to a classical model from statistical physics, known as
the Curie--Weiss model~\cite{EllNew78,Ell85}, which captures a
phenomenon known as spontaneous magnetization.  It defines a family of
distributions over the spin-hypercube $\{-1, +1 \}^\usedim$ via the
probability mass function
\begin{align}
\label{EqnCurieWeiss}
\ProbZ(z_1, \ldots, z_\usedim; \beta) & \propto \exp \big\{
\frac{\beta}{2 \usedim} \big(\sum_{i=1}^{\usedim} z_i \big)^2 \big\},
\end{align}
where $\beta \in \real$ is a parameter.  For physical reasons, the sum
$\sum_{i=1}^d z_i$ is referred to as the magnetization.  At $\beta =
0$, we have a product distribution, whereas as $\beta \rightarrow
-\infty$, the model assigns increasing mass to balanced configurations
with $\sum_{i=1}^d z_i = 0$.  On the other side, as $\beta$ increases
along the positive line, it is known to have a phase transition at
$\beta = 1$.  (In particular, for $\beta < 1$, the magnetization is
symmetrically distributed around zero, whereas for $\beta > 1$, its
distribution becomes approximately bimodal, with mass concentrated
near two oppositely magnetized states.)
\begin{figure}[htb!]
\begin{center}
  \begin{tabular}{@{}cccc@{}}
    \widgraph{0.22\textwidth}{\figdir/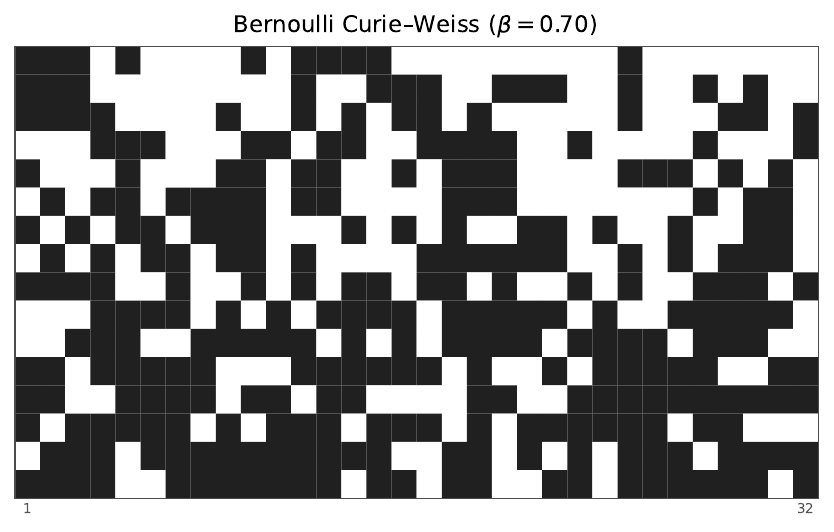}
    &
    \widgraph{0.22\textwidth}{\figdir/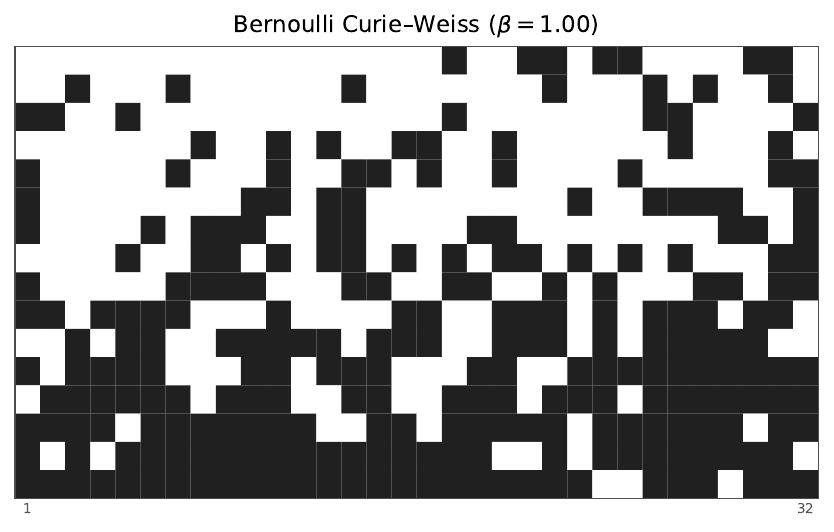}
    &
    \widgraph{0.22\textwidth}{\figdir/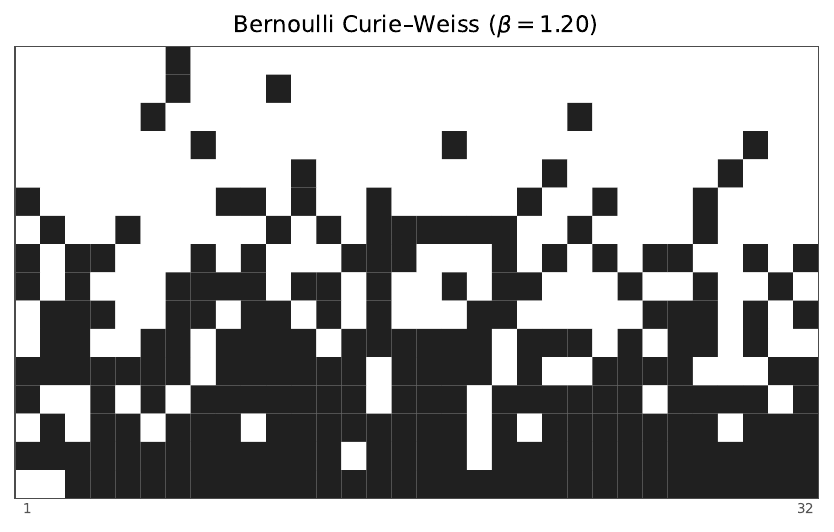}
    &
    \widgraph{0.22\textwidth}{\figdir/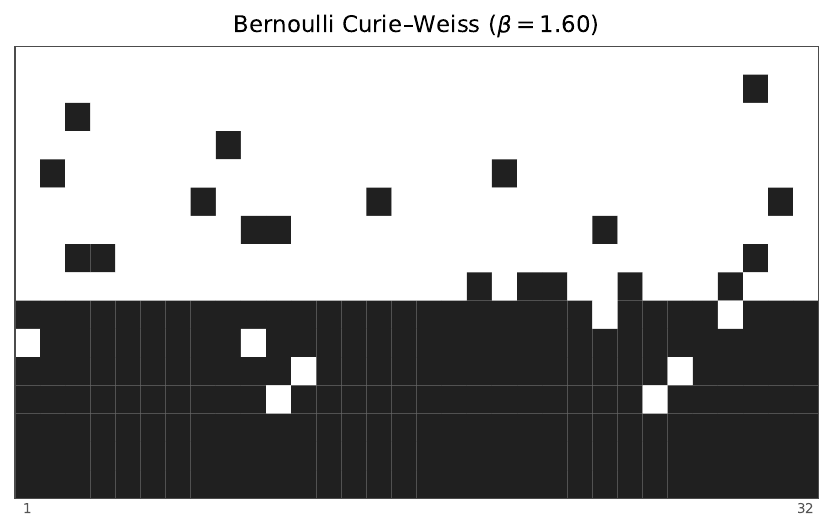}
    \\
    \widgraph{0.25\textwidth}{\figdir/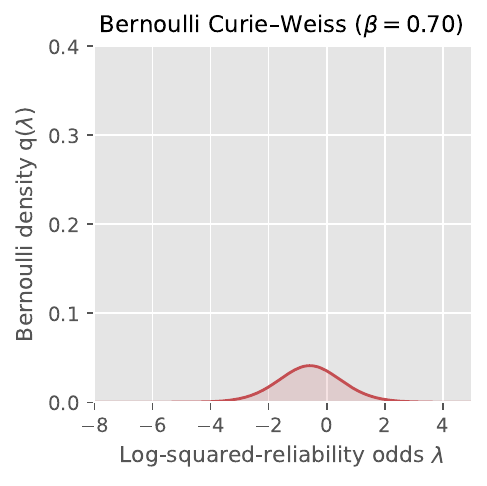} &
    \widgraph{0.25\textwidth}{\figdir/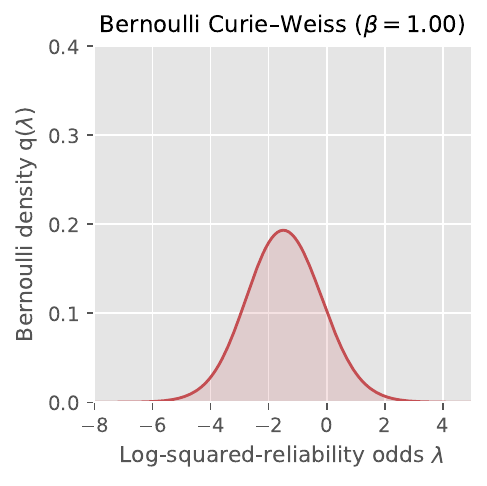} &
    \widgraph{0.25\textwidth}{\figdir/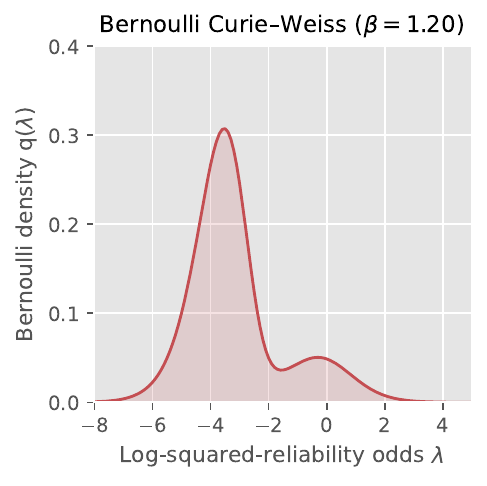} &
    \widgraph{0.25\textwidth}{\figdir/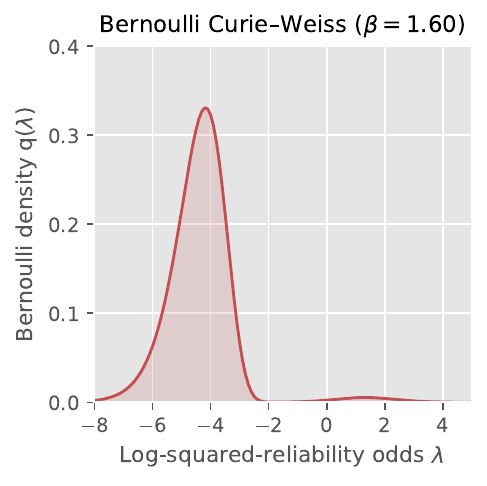}
\end{tabular}
\caption{Illustration of the $\igc$ density for the Curie--Weiss
  model~\eqref{EqnCurieWeiss}.  Bottom panels from left to right
  correspond to $\beta \in \{0.70, 1.00, 1.20, 1.60 \}$.  For each
  choice of $\beta$, the matrix presents a collection of random
  samples $Z \sim \ProbZ$, one in each row.}
\label{FigCurieWeiss}
\end{center}
\end{figure}
Finally, for large positive values of $\beta$, the distribution
becomes close to bimodal, with nearly all mass aligned on states with
absolute magnetization $|\sum_{i=1}^d z_i| = d$.

In~\Cref{FigCurieWeiss}, we plot the log-\SR density $\qdens$ from
equation~\eqref{EqnDefnQdens} associated with the
model~\eqref{EqnCurieWeiss} for four different choices of $\beta \in
\{0.70, 1.00, 1.20, 1.60 \}$.  Above each $\qdens$-plot, we show a
$\numobs \times \usedim$-matrix of samples, where each of the
$\numobs$ rows corresponds to a sample $Z \in \{-1, 1\}^\usedim$ drawn
from the associated model.  These plots illustrate how the \igc
density reflects the changing dependence structure as a function of
the \mbox{Curie--Weiss} parameter $\beta$.  In the subcritical regime,
illustrated by $\beta = 0.70$, the samples exhibit relatively weak
dependence, and the associated \igc density is correspondingly modest.
As $\beta$ approaches and crosses the critical value $\beta = 1$, the
dependence strengthens and the \igc density develops a pronounced
mode; at the same time, the samples begin to separate into two groups
with opposite magnetizations.  For $\beta = 1.20$, the \igc density
becomes bimodal, suggesting two distinct scales of dependence along
the diffusion path: the global choice between the two magnetized
phases, and the residual fluctuations of configurations within each
phase.  Finally, for the larger value $\beta = 1.60$, the \igc density
again becomes approximately unimodal.  At this point, the within-phase
fluctuations are much smaller, leaving the global choice between the
two strongly magnetized components as the dominant source of
dependence.


\subsubsection{Hierarchical mixture model}

As our final example, we revisit a model class first introduced in our
previous paper~\cite{Wai_UGC} on masking complexity.  It involves a
random vector $Z \in \{0,1 \}^\usedim$ over the Boolean hypercube that
places mass on $2^{L}$ prototype configurations for an integer $L \ll
\usedim$.  More interestingly, these binary prototypes are related via
a latent tree structure, and the integer $L$ corresponds to the number
of levels in the tree.

\begin{figure}[htb!]
\begin{center}
\begin{tabular}{@{}cc@{}}
\widgraph{0.60\textwidth}{\figdir/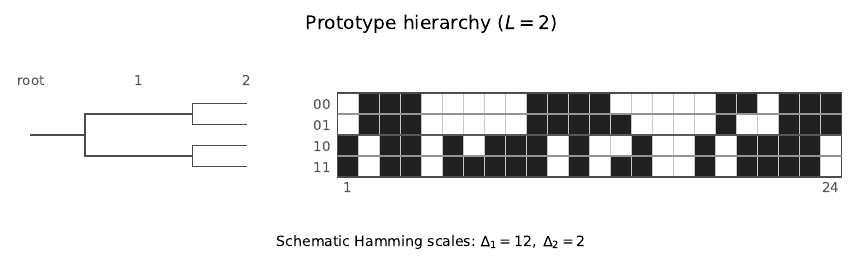}
&
\widgraph{0.28\textwidth}{\figdir/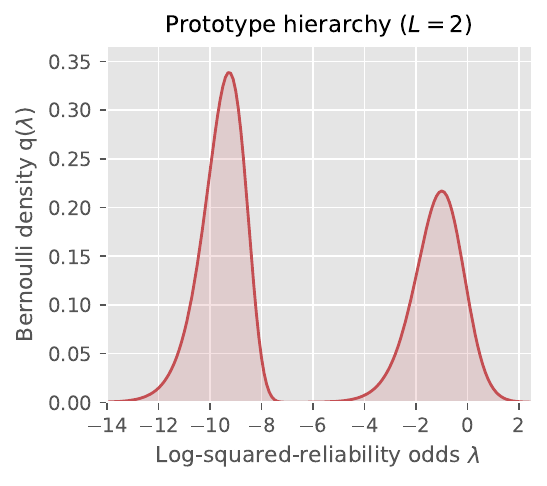}
\\
(a) & (b) \\
\widgraph{0.60\textwidth}{\figdir/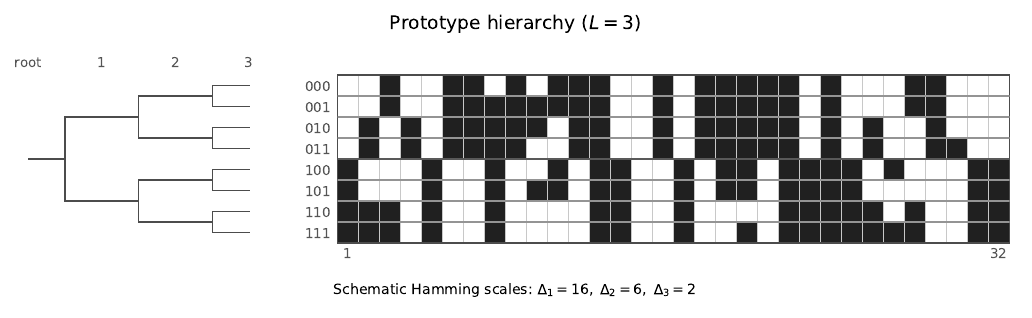}
&
\widgraph{0.28\textwidth}{\figdir/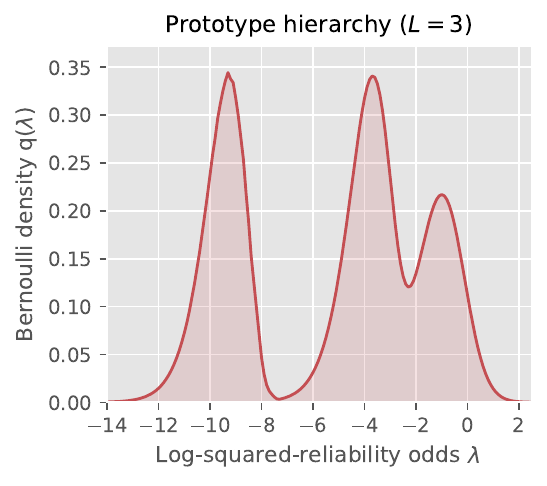}
\\
(c) & (d) \\
%
%
\end{tabular}
\caption{Plots of the \igc densities $\qdens$ associated with various
  types of hierarchical binary-prototype mixtures.  The rows correspond
  to hierarchy depths $L = 2$ and $L = 3$.  Left panels (a) and (c)
  show the leaf prototypes, while right panels (b) and (d) show the
  corresponding densities $\qdens$.}
\label{FigHierarchicalMixture}
\end{center}
\end{figure}

In terms of the shorthand $M = 2^L$, let $\{V^1, \ldots, V^M \}$
denote the binary prototype vectors.  Each prototype is defined by a
path down a binary tree of depth $L$; see the left column
of~\Cref{FigHierarchicalMixture} for some examples.  The geometry of
the prototypes is related to the tree structure in the following way.
For any pair $V^a$ and $V^b$ of prototypes, suppose that they first
diverge at level $\ell \in \{1, \ldots, L\}$ of the tree; our
construction then ensures that $V^a$ and $V^b$ are separated by
Hamming distance $\Delta_\ell$, for some pre-specified set of
distances $\{\Delta_\ell \}_{\ell =1}^{L}$.  This leads to a total
dimension $d = \sum_{\ell=1}^{L} \Delta_\ell$.  Panel (a)
of~\Cref{FigHierarchicalMixture} shows a toy example with $L = 2$,
$(\Delta_1, \Delta_2) = (12, 2)$, and dimension $d = 12 + 2 = 14$.
This set-up leads to $M = 2^L = 4$ prototype vectors in the Boolean
hypercube $\{0,1\}^{14}$.  To track them, we assigned the index
labels $\{00, 01, 10, 11 \}$, plotted as rows in the corresponding
black-white matrix.  Observe that the pair of prototypes indexed by
$00$ and $11$ diverge at level $\ell = 1$ of the tree, and their
Hamming distance is equal to $\Delta_1 = 12$.  Similarly, the
prototypes indexed by $00$ and $01$ diverge at level $\ell = 2$ of the
tree, and have Hamming distance $\Delta_2 = 2$.  Panel (c) illustrates
a larger construction with $L = 3$, still in a toy dimension $d = 34$
for illustrative purposes.

Panels (b) and (d) show the \igc densities for $L = 2$ and $L = 3$
models, using larger separations and dimensions: the $L = 2$ model has
dimension $\usedim = 32776$ and Hamming separations $(\Delta_1,
\Delta_2) = (32768, 8)$, whereas the $L = 3$ model has dimension
$\usedim = 33288$, and Hamming separations $(\Delta_1, \Delta_2,
\Delta_3) = (32768, 512, 8)$.  The \igc density
for the $L = 2$ model has two distinct peaks in
log-\SR-odds, and these two peaks have a very natural interpretation.
The prototypes at level $\ell = 1$ differ by Hamming distance
$\Delta_1 = 32768$, so their ambiguity is resolved at a fairly early
denoising time.  Only at a later time does the second peak appear,
corresponding to resolving the ambiguity between pairs separated by
Hamming distance $\Delta_2 = 8$. Each of the $L = 3$ peaks for the \igc
density in panel (d) has an analogous interpretation.


\section{Product-reference diffusion and KL guarantees}
\label{SecMain}

In this section, we begin in~\Cref{SecAlg} by describing the class of
sampling algorithms that we analyze.  They are based on product
approximations to the transition kernel~\eqref{EqnKerExact}, and can
be parameterized either in terms of the single-site denoisers or
their leave-one-out versions.  In~\Cref{SecKLBound}, we present a
general upper bound on the Kullback--Leibler (KL) error of the
sampler: it involves increments of the \igc density, combined with
prefactors in terms of the squared-reliability (SR) odds ratio.  This
geometric factor is key, as it identifies stepsizes with constant
spacing in log-\SR-odds as canonical.  We derive a corollary for this
simple stepsize choice. It gives guarantees in terms of the $\TC/\DTC$
measures, and as discussed in~\Cref{SecReference}, it also provides
insight into the choice of reference distribution $\refdist$.
Finally, \Cref{SecBivQdens} is devoted to an exploration of the
bivariate \igc kernel, which gives exact relations for the KL error,
the univariate \igc upper bound, and the $\TC/\DTC$ measures.

\subsection{Sampling using product-reference diffusion (\prd)}
\label{SecAlg}

For $0 < a < b < 1$, we let $\Ker_{a, b}(\cdot \mid x) \defn \Law(X_b
\mid X_a = x)$ denote the exact kernel of the reference-to-output
process, as previously defined in equation~\eqref{EqnKerExact}.  A
standard approximation is based on the product kernel
\begin{align}
  \label{EqnProductKernel}
  \Ktilde{a}{b}(y \mid x) & \defn \prod_{\jind = 1}^d \Prob_i^{a \to
    b}(y_i \mid x),
\end{align}
where $\Prob_i^{a \to b}(y_i \mid x) \defn \Prob(X_{b,i} = y_i \mid X_a
= x)$.

\paragraph{Different parameterizations:}
The transition kernel $\Prob_i^{a \to b}$ can be parameterized either
via the single-site posterior $\Den_{i,a}(w, x) \defn \Prob(Z_i = w
\mid X_a = x)$, or via its leave-one-out version $\DenLoo_{a,i}(w,
x_{-i}) \defn \Prob(Z_i = w \mid X_{a,-i} = x_{-i})$.  Here $X_{a,-i}$
denotes the vector obtained from $X_a$ by removing coordinate $i$.
For the ordinary posterior, conditioning on $Z_i$ gives the one-site
bridge representation
\begin{subequations}
\begin{align}
  \label{EqnDenoiserBridge}
  \Prob_i^{a \to b}(y \mid x) & = \sum_{w \in \Alphabet}
  \BridgeAB_{i}(y \mid x_i, w) \Den_{i,a}(w, x), \qquad \mbox{where}
  \\
  \label{EqnDefnOneSiteBridge}
  \BridgeAB_i(y \mid x_i, w) & \defn \Prob(X_{b,i} = y \mid X_{a,i} =
  x_i, Z_i = w) = \frac{\Refdist_{i,b}(y \mid w) \Refdist_{i,a/b}(x_i
    \mid y)} {\Refdist_{i,a}(x_i \mid w)}.
\end{align}
\end{subequations}
Here $\Refdist$ was defined in equation~\eqref{EqnGeneralRef}.
See~\Cref{AppLOO} for the parameterization in terms of the LOO
posterior $\DenLoo$.


\paragraph{\prd sampling updates:}

Any product-reference diffusion (\prd) sampler operates on a
collection of single-site transitions
$\{ \Prob_i^{a \to b}, i = 1, \ldots, \usedim \}$ for each
squared-reliability time pair $0 < a < b < 1$.  Given an integer
iteration number $N \geq 2$, it takes as input a grid of the form
\begin{align}
  \label{EqnGrid}
  0 < \time_0 < \time_1 < \cdots < \time_N = \tfinal < 1.
\end{align}
The sampler then proceeds through the following three steps:
\begin{center}
\begin{tcolorbox}[
  width=0.95\linewidth,
  title={Product-reference diffusion sampler},
  colbacktitle=blue!18,
  coltitle=black,
  fonttitle=\bfseries,
  toptitle=1.5mm,
  bottomtitle=1.5mm,
  colback=blue!6,
  colframe=blue!45!black,
  boxrule=0.5pt,
  arc=2mm,
  boxsep=0pt,
  left=2mm,
  right=2mm,
  top=1.5mm,
  bottom=1.5mm,
  before skip=0pt,
  after skip=0pt,
  before upper={\setlength{\parskip}{0pt}%
    \setlength{\abovedisplayskip}{4pt}%
    \setlength{\belowdisplayskip}{4pt}%
    \setlength{\abovedisplayshortskip}{3pt}%
    \setlength{\belowdisplayshortskip}{3pt}}
]
\textbf{1. Initialization.}
At squared reliability $\time_0$, draw
\begin{subequations}
\begin{align}
  \label{EqnPRDInit}
  \Xhat^0 & \sim \bigotimes_{i = 1}^{\usedim} \Law(X_{\time_0, i}).
\end{align}

\textbf{2. Product updates.}
For $k = 1, \ldots, N$, update the coordinates independently
according to
\begin{align}
  \label{EqnPRDUpdate}
  \Law(\Xhat^k \mid \Xhat^{k - 1} = x)
  & \defn \bigotimes_{i = 1}^{\usedim}
  \Prob_i^{\time_{k - 1} \to \time_k}(\mathord\cdot \mid x).
\end{align}

\textbf{3. Completion.}
From $\tfinal$ to the clean endpoint, draw
\begin{align}
  \label{EqnPRDCompletion}
  \Law(\Xhat^{N + 1} \mid \Xhat^N = x)
  & \defn \bigotimes_{i = 1}^{\usedim}
  \Law(Z_i \mid X_{\tfinal} = x), \qquad \mbox{and return $\Zhat = \Xhat^{N + 1}$.}
\end{align}
\end{subequations}
\end{tcolorbox}
\end{center}
Thus, for a fixed iteration budget $N$, the key design parameter---and
one for which our theory provides guidance---is the design of the stepsize
grid~\eqref{EqnGrid}.  \Cref{ThmMaster} provides an upper bound on the
KL error for any grid, whereas~\Cref{CorSingleBlock} specializes to
an equi-spaced grid.  Our later theory in~\Cref{SecMultiBlock}
develops adaptive multi-block grids.

\subsection{Controlling the Kullback--Leibler error}
\label{SecKLBound}

A key question is to understand the accuracy of the final output
$\Zhat = \Xhat^{N + 1}$ relative to the target distribution.  A standard measure
of accuracy is the Kullback--Leibler (KL) divergence given by
\begin{align}
\KL(\ProbZ \| \Prob_{\Zhat}) & \defn \sum_{z \in \Alphabet^d}
\ProbZ(z) \log \Big( \frac{\ProbZ(z)}{\Prob_{\Zhat}(z)} \Big).
\end{align}
We begin in~\Cref{SecGeneral} with a general upper bound that applies
to any stepsize grid~\eqref{EqnGrid}.  In~\Cref{SecSingleBlock}, we
specialize it to stepsizes that are equi-spaced in the log-\SR-odds
parameter, and derive concrete results in terms of the aggregate \igc
complexity.

\subsubsection{A general guarantee}
\label{SecGeneral}

The main result of this section provides a KL error upper bound that
involves the \igc increments
\begin{align}
  \label{EqnDefnGinfoTwo}
  \Ginfo(a, b) & \defn \int_a^b r (1 - r) \gfun(r) \, dr.
\end{align}
associated with a given interval.
\mygraybox{
\begin{theorem}
\label{ThmMaster}
For any grid~\eqref{EqnGrid} of the interval $[\time_0, \tfinal]$, the
\prd sampler returns output $\Zhat$ such that
\begin{align}
  \label{EqnMaster}
  \KL( \Prob_Z \| \Prob_{\Zhat}) & \leq 2 \; \sum_{j = 0}^{N - 1} \big \{
  \frac{\odds(\time_{j + 1})}{\odds(\time_j)} - 1 \big \}
  \Ginfo(\time_j, \time_{j + 1}) + \BOUNDARY,
\end{align}
where $\odds(t) \defn \frac{t}{1 - t}$ is the odds of the squared
reliability $t$.
\end{theorem}
}
The bound~\eqref{EqnMaster} involves initialization and termination
costs defined in terms of the total correlation, and a conditional
version thereof.  Letting $\Ent$ denote the (Shannon) entropy, the
initialization cost is given by
\begin{align}
\TC(X_{\rstar}) & \defn \big \{\sum_{i=1}^d \Ent(X_{\rstar,i}) \big \} -
\Ent(X_{\rstar}),
\end{align}
which corresponds to the KL divergence between the joint
distribution of $X_{\rstar}$ and the product of its marginals,
consistent with our choice of initialization~\eqref{EqnPRDInit}.
The termination cost $\TC(Z \mid X_{\rfinal})$ arises from the
completion step~\eqref{EqnPRDCompletion}.

These two costs are straightforward to control by suitable choices of
the pair $r_0$ and $\rfinal$.  For the symmetric choice $\rfinal = 1 -
\rstar$, we are guaranteed to have
\begin{align}
\label{EqnBoundaryBound}  
\BOUNDARYSPEC & \leq c \, \rstar \big( \usedim |\Alphabet| \big)^3
\qquad \mbox{for a universal constant $c > 0$,}
\end{align}
so that setting $\rstar = (\usedim |\Alphabet|)^{-k}$ for $k \geq 4$
drives this error to zero at rate $(\usedim
|\Alphabet|)^{k-3}$. Importantly, the sampler complexity is only
logarithmic in $\rstar$.  Consequently, the bulk of our analysis
focuses on characterizing the KL discretization sum in the
bound~\eqref{EqnMaster}.

\paragraph{Learned posterior distributions:}
The statement of~\Cref{ThmMaster} uses the exact single-site
posteriors, whereas in practice, these oracle objects are replaced by
versions that are estimated from data.
Here we describe how replacing the ordinary posterior
$\Den_{i,a}$ with a learned version $\DenHat_{i,a}$ leads to an additional
error term in the upper bound~\eqref{EqnMaster}.
To be clear,
the LOO parameterization can be handled in
exactly the same way using the corresponding bridge
from~\Cref{AppLOO}.

For each coordinate $i = 1, \ldots, \usedim$ and time $a$, let
$\DenHat_{i,a}$ be an estimate of the true one-site posterior
$\Den_{i,a}$.  Passing this estimate through the one-site
bridge~\eqref{EqnDenoiserBridge} yields a \emph{learned one-site
kernel}
\begin{subequations}
\begin{align}
\label{EqnLearnedKernel}  
  \Khat{a}{b}^{(i)}(y \mid x) \defn \sum_{w \in \Alphabet}
  \BridgeAB_i(y \mid x_i, w) \DenHat_{i,a}(w \mid x).
\end{align}
This collection of learned one-site kernels defines the corresponding
learned product kernel $\Khat{a}{b}(\cdot \mid x) \defn \bigotimes_{i
  = 1}^{\usedim} \Khat{a}{b}^{(i)}(\cdot \mid x)$.

By the bridge representation~\eqref{EqnDenoiserBridge}, the true
one-coordinate marginal of $\Ker_{a,b}(\cdot \mid x) \defn \Law(X_b
\mid X_a = x)$ is precisely $\Prob_i^{a \to b}(\cdot \mid x)$.
Consequently, the approximation $\Ktilde{a}{b}(\cdot \mid x)$ is the
product of the true one-coordinate marginals. Via the KL projection
identity for product distributions, we are guaranteed to have
\begin{align*}
\KL\big( \Ker_{a,b}(\cdot \mid x) \,\big\|\, \Khat{a}{b}(\cdot \mid x)
\big) & = \KL\big( \Ker_{a,b}(\cdot \mid x) \,\big\|\,
\Ktilde{a}{b}(\cdot \mid x) \big) \quad + \quad \sum_{i = 1}^{\usedim}
\KL\big( \Prob_i^{a \to b}(\cdot \mid x) \,\big\|\, \Khat{a}{b}^{(i)}(\cdot
\mid x) \big).
\end{align*}
Thus learning the posterior contributes an additive error on top of
the exact-posterior discretization error.  Averaging over $X_a$,
define the block denoiser penalty
\begin{align}
  \label{EqnDenoiserError}
\DenError_{\mathrm{den}}(a, b) & \defn \sum_{i = 1}^{\usedim}
\Exs_{X_a}\left[ \KL\left( \Prob_i^{a \to b}(\cdot \mid X_a)
  \,\middle\|\, \Khat{a}{b}^{(i)}(\cdot \mid X_a) \right) \right],
\end{align}
\end{subequations}
which is then added to the upper bound~\eqref{EqnMaster}
in~\Cref{ThmMaster}.  Note that the estimated posteriors
$\DenHat_{i,a}$ enter this expression via the bridge, as in
equation~\eqref{EqnLearnedKernel}.


\subsubsection{Single block consequence and KL iteration complexity}
\label{SecSingleBlock}

We now derive a consequence of~\Cref{ThmMaster} that applies to a
single-block sampler, meaning one that uses a non-adaptive stepsize
schedule across the entire interval.  Here we specialize to the
symmetric interval $\IntStar \defn [\rstar, 1 - \rstar]$, which leads
to the log-\SR-odds image $[-\Elld, \Elld]$ where $\Elld \defn
\log\left( \frac{1 - \rstar}{\rstar} \right)$.  In particular, it
applies to an $N$-step grid that is uniform in log-\SR-odds:
\begin{align}
  \label{EqnSingleBlockGrid}
  \logodds(r_j) & = -\Elld + \frac{2 \Elld j}{N}, \qquad j = 0, 1,
  \ldots, N.
\end{align}
The following result summarizes the resulting guarantee: \mygraybox{
  \begin{corollary}
    \label{CorSingleBlock}
    For any choice of $\rstar \in (0, 1/2)$ and iteration number $N
    \geq 2 \Elld$, applying the \prd sampler with the stepsize
    protocol~\eqref{EqnSingleBlockGrid} over the interval $[\rstar, 1
      - \rstar]$ returns output $\Zhat$ such that
\begin{align}
  \label{EqnSingleBlockBound}
  \KL(\Prob_Z \| \Prob_{\Zhat}) & \leq \frac{8 \Elld}{N}
  \Ginfo(\rstar, 1 - \rstar) + \BOUNDARYSPEC \qquad \mbox{where
    $\Elld = \log\left( \frac{1 - \rstar}{\rstar} \right)$.}
\end{align}
  \end{corollary}
}
\begin{proof}
Since every step has the same log-\SR-odds multiplier,
\Cref{ThmMaster} and the additivity of $\Ginfo$ give
\begin{align*}
  \KL(\Prob_Z \| \Prob_{\Zhat}) & \leq 2 \left\{ \exp\left(
    \frac{2 \Elld}{N}\right) - 1 \right\} \Ginfo(\rstar, 1 -
    \rstar) + \BOUNDARYSPEC.
\end{align*}
Using the elementary bound $e^x - 1 \leq 2 x$ for $x \in [0, 1]$ and
the assumed lower bound $N \geq 2 \Elld$ yields the claim.
\end{proof}

Let us now explore the iteration complexity of the single-block \prd
sampler, meaning the number of iterations $N(\varepsilon)$ required to
achieve KL error at most $\varepsilon$.  In particular, suppose that
we use the symmetric interval $[\rstar, 1 - \rstar]$ with $\rstar
\defn \frac{\varepsilon}{2 c \big( \usedim |\Alphabet| \big)^3}$.
Then inequality~\eqref{EqnBoundaryBound} gives $\BOUNDARYSPEC \leq c
\, \rstar \big( \usedim |\Alphabet| \big)^3 \leq \varepsilon/2$.
The \prd sampler, when
implemented with the equi-spaced grid~\eqref{EqnSingleBlockGrid},
yields an output $\Zhat$ such that $\KL(\ProbZ \| \Prob_{\Zhat}) \leq
\varepsilon$ using at most
\begin{subequations}
\begin{align}
\label{EqnIterationComplexity}
  N(\varepsilon) & = c_0 + c_1 \frac{ \Ginfo(\rstar, 1 -
    \rstar)}{\varepsilon} \log \Big( \frac{\usedim
    |\Alphabet|}{\varepsilon} \Big)\quad
\end{align}
iterations, for universal constants $(c_0, c_1)$.  Here we have used
the fact that $\log(1/\rstar) = \Order\big(\log(\frac{\usedim
  |\Alphabet|}{\varepsilon})\big)$ for our specified choice of
$\rstar$.  Recalling the definition~\eqref{EqnDefnQdens} of the \igc
density $\qdens$, we have
\begin{align}
\Ginfo(\rstar, 1 - \rstar) & = \int_{-\Elld}^{\Elld}
\qdens(\lam) \, d\lam,
\end{align}
which justifies the single-block scaling~\eqref{EqnDefnSingBlock}
asserted in~\Cref{SecOverview}.

As we show in the next section (see~\Cref{CorUpperTC}), the aggregate
\igc complexity is upper bounded as
 \begin{align}
   \label{EqnUpperTCNew}
\Ginfo(\rstar, 1 - \rstar) \; \leq \Ginfo(0,1) \, & \leq 2 \; \min \{
\TC(Z), \; \DTC(Z) \},
 \end{align}
\end{subequations}

As an important consequence, the single-block sampler inherits any
bound that can be obtained for $\min \{ \TC(Z), \DTC(Z) \}$.  These
multivariate dependence measures adapt to the structure of $Z$.
Notably, in their analysis of a CTMC masking sampler, Dmitriev et
al.~\cite{DmiEtAl26} explored a diverse range of examples for which
the min-$\TC/\DTC$ value can be computed or bounded in terms of
low-dimensional latent structure.  Among other interesting examples,
they analyzed the $\TC/\DTC$ structure of hidden Markov models with
slowly varying latent states, stochastic block models, and quantized
latent models.  Thus, by combining our bounds with their analysis, we
can assert that similar low-dimensional guarantees hold for \prd
samplers.

For their masking-related functional, called the effective total
correlation, Dmitriev et al.~\cite{DmiEtAl26} also observed that the
min-$\TC/\DTC$ bound can be very conservative.  In particular, via a
construction using independent blocks, they gave an example where
there is an $\Omega(d)$ gap between their complexity measure and the
min-$\TC/\DTC$ bound.  We can adapt their argument to the current setting,
using our aggregate \igc measure $\Ginfo(0,1)$ as the base, to show
that the same gap persists.


\subsubsection{Effects of varying the reference $\refdist$}
\label{SecReference}

The choice of reference distribution $\refdist$ is an important design
choice, and past empirical
work~\cite{AusEtAl21,VigEtAl23,LaaEtAl25,SchEtAl25} has shown that
this choice can have substantial effects on sampler performance.  Here
we provide theoretical insight into why the choice of reference
matters, and when it can actually lead to dimension-dependent
differences between different \prd samplers.

\begin{figure}[h]
\begin{center}
  \begin{tabular}{ccc}
    \widgraph{0.33\textwidth}{\figdir/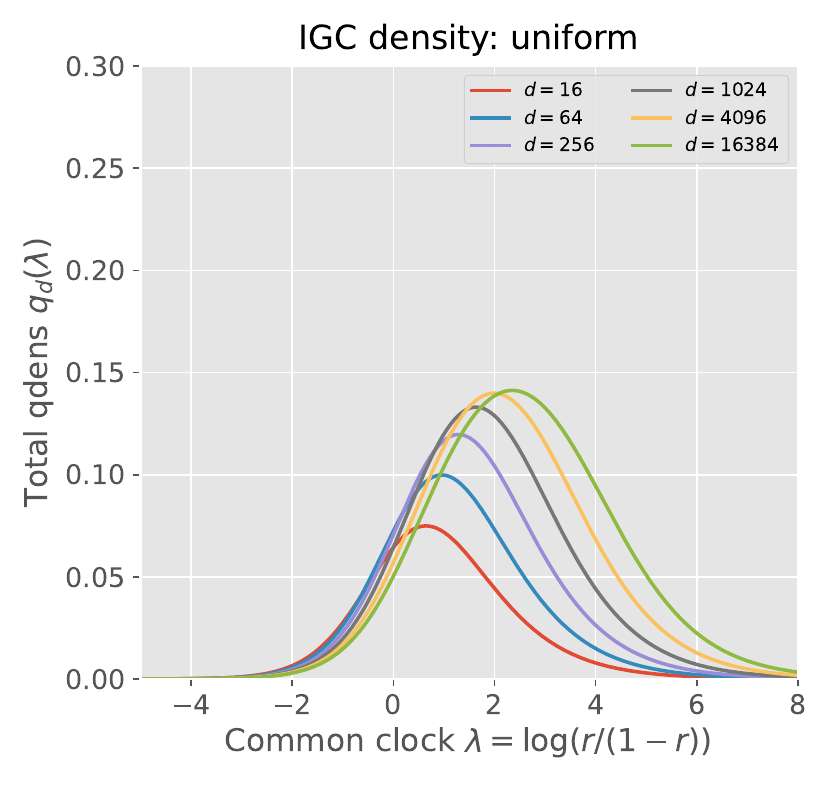}
    &
    \widgraph{0.33\textwidth}{\figdir/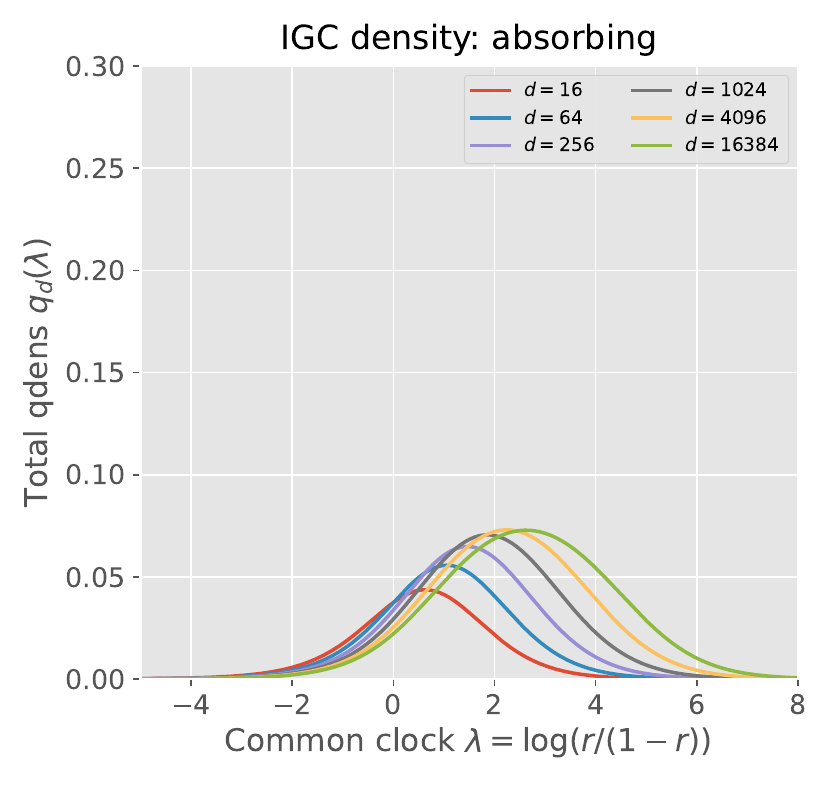}
    &
    \widgraph{0.33\textwidth}{\figdir/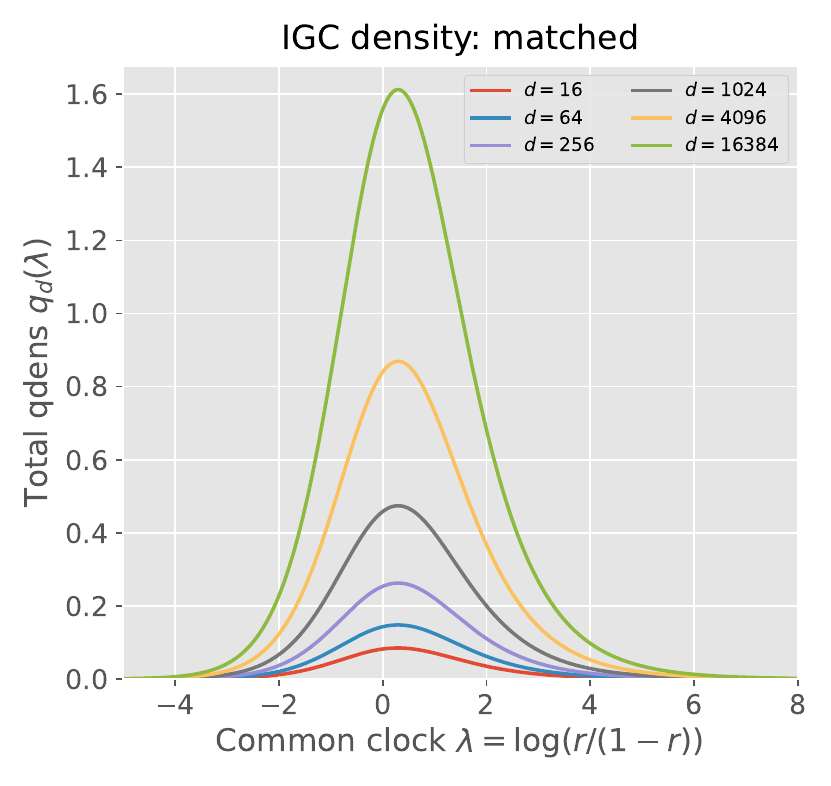}
    \\
    (a) & (b) & (c)
  \end{tabular}
\caption{Log-odds \igc densities $\qdens$ for Ensemble~A, with
  $\epsilon_\usedim = \usedim^{-1/2}$, under (a) the uniform reference, (b) the
  absorbing-state reference, and (c) the marginal-matched reference.}
\label{FigForbidQdens}
\end{center}
\end{figure}

More precisely, we consider three different choices: the
\emph{uniform} reference; the \emph{absorbing-state} reference; and
the \emph{marginal-matched} reference, in which $\refdist$ corresponds
to a product of the marginals of $\ProbZ$.  For each of these three
choices, we use the \igc representation and an associated family of
binary distributions to establish two important facts:
\begin{itemize}[leftmargin=1em, topsep=2pt, itemsep=2pt,
  parsep=0pt, partopsep=0pt]
  \item The choice of reference distribution can have a
    dimension-dependent effect on the iteration complexity of \prd
    samplers.
  \item There is no universal performance ordering among different
    reference distributions.
\end{itemize}

Let us begin by describing a general family of distributions over a
binary random vector $(Z_1, \ldots, Z_\usedim) \in \{0,1 \}^\usedim$.
Suppose that the dimension $\usedim$ is even, and we partition the
coordinates into $\usedim/2$ disjoint pairs.  The pairs are
independent and identically distributed, with each block $(Z_{2k-1},
Z_{2k})$ drawn from the bivariate PMF $P_\epsilon$ on $\{0, 1\}^2$
given by
\begin{align*}
  P_\epsilon(0, 0) & = \epsilon, & P_\epsilon(0, 1) & = 0, &
  P_\epsilon(1, 0) & = \frac{1 - \epsilon}{2}, & P_\epsilon(1, 1) & =
  \frac{1 - \epsilon}{2}.
\end{align*}
Thus, the configuration $(0, 1)$ is forbidden within every pair.  The
one-coordinate marginals are $\Ber(1 - \epsilon)$ and $\Ber((1 -
\epsilon)/2)$ for the first and second coordinates of each pair,
respectively.

\begin{figure}[h]
\begin{center}
  \begin{tabular}{cc}
    \widgraph{0.45\textwidth}{\figdir/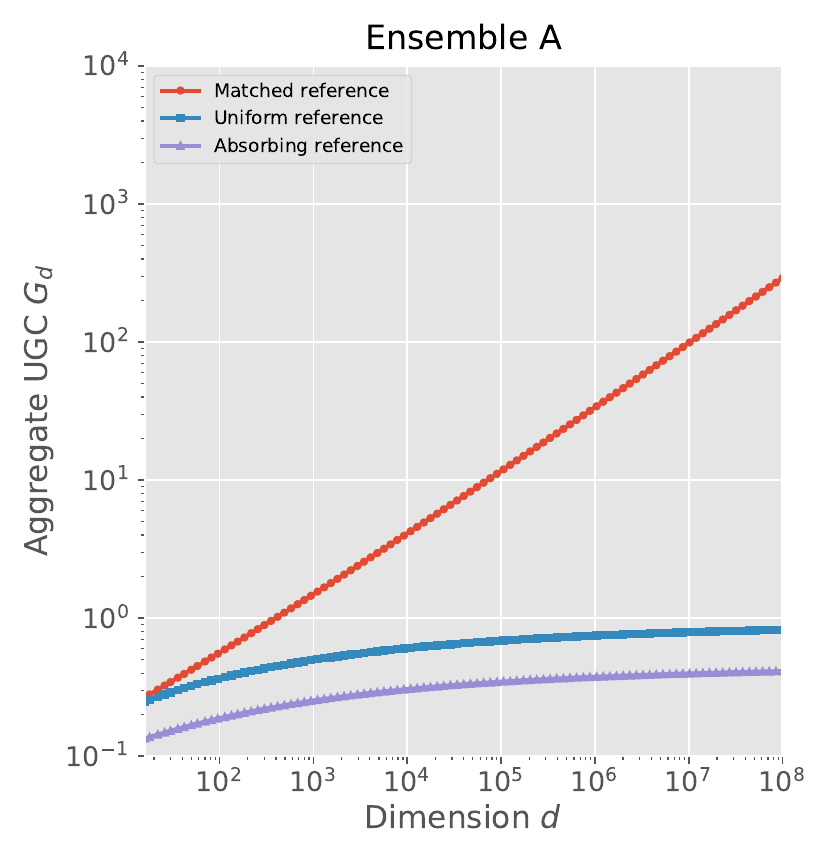} &
    \widgraph{0.45\textwidth}{\figdir/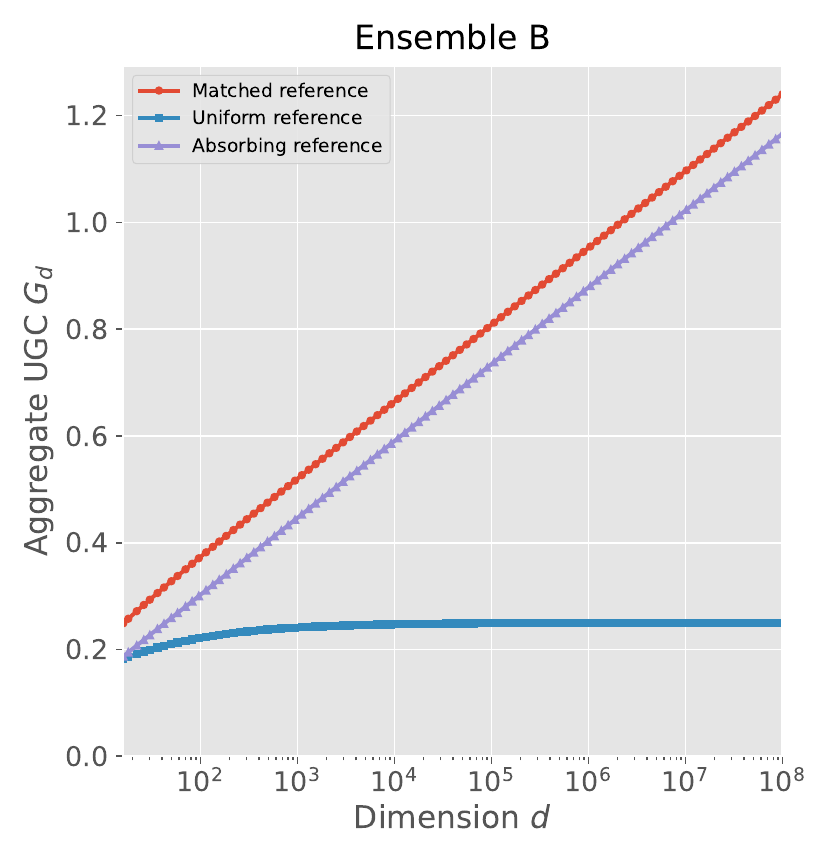}
    \\
    (a) & (b)
  \end{tabular}
\caption{Dimension dependence of the aggregate \igc for different
  choices of reference.  Panel~(a) shows Ensemble~A, with
  $\epsilon_\usedim = \usedim^{-1/2}$, for which the marginal-matched
  reference becomes increasingly unfavorable.  Panel~(b) shows
  Ensemble~B, with $1 - \epsilon_\usedim = \usedim^{-1}$, for which
  the uniform reference has constant-order \igc whereas the
  \igc for the absorbing-state reference grows as $\Theta(\log \usedim)$.}
\label{FigForbidScale}
\end{center}
\end{figure}

For our three choices of reference distribution, when specialized to
this ensemble, the marginal-matched choice is $\nu_1 = \Ber(1 -
\epsilon)$ and $\nu_2 = \Ber((1 - \epsilon)/2)$, the uniform choice is
$\nu_1 = \nu_2 = \Ber(1/2)$, and the absorbing-state choice is $\nu_1 = \nu_2 =
\delta_0$.  We use the same family $P_\epsilon$ to exhibit both
dimension-dependent scaling, and no uniform ordering among the
choices.

\paragraph{Ensemble A:}
For the first ensemble, we take $\epsilon_\usedim \defn
\usedim^{-1/2}$.  With this choice, the exceptional configuration $(0,
0)$ becomes increasingly rare with dimension, while the first marginal
approaches a point mass at one and the second approaches $\Ber(1/2)$.
This is the parameter regime used for the \igc densities
in~\Cref{FigForbidQdens} and for the scaling experiment in panel~(a)
of~\Cref{FigForbidScale}.  In this regime, the marginal-matched
reference is the unfavorable choice, with numerical results showing
its aggregate \igc growing as $\Theta(\sqrt{d})$ in dimension $d$
relative to the uniform and absorbing-state references.

\paragraph{Ensemble B:}
For the second ensemble, we instead approach the absorbing state.
Writing $\delta_\usedim \defn 1 - \epsilon_\usedim$, we take
$\delta_\usedim \defn \frac{1}{\usedim}$ and $\epsilon_\usedim = 1 -
\frac{1}{\usedim}$, so that the bivariate PMF takes the form
\begin{align*}
  P_{\epsilon_\usedim}(0, 0) & = 1 - \frac{1}{\usedim}, &
  P_{\epsilon_\usedim}(0, 1) & = 0, & P_{\epsilon_\usedim}(1, 0) & =
  \frac{1}{2 \usedim}, & P_{\epsilon_\usedim}(1, 1) & = \frac{1}{2
    \usedim}.
\end{align*}
The target therefore concentrates around the absorbing configuration
$(0, 0)$ as $\usedim$ grows.  Nevertheless, the absorbing-state
reference is substantially less favorable than the uniform reference:
as shown in panel~(b) of~\Cref{FigForbidScale}, the aggregate \igc for
the uniform reference remains of constant order, whereas that of the
absorbing-state reference exhibits a growth rate of $\Theta(\log
\usedim)$.


\subsection{The bivariate \igc kernel and its structure}
\label{SecBivQdens}

As noted in the introduction, our analysis of \prd sampling differs
from our previous analyses of Gaussian and masking diffusion
samplers~\cite{Wai_DGC,Wai_UGC} in an important way.  In particular,
while the latter samplers have a complexity structure that is
essentially univariate in nature, the \igc setting leads to a
bivariate structure.  This section is devoted to an exploration of the
bivariate \igc kernel that underlies our analysis.

Let $\newrevinv(\eta) = \frac{e^\eta}{1 + e^\eta}$ be the
transformation from log-\SR-odds time back to squared reliability
time.  For each coordinate $i = 1, \ldots, \usedim$, we define the
log-\SR-odds cut mutual information
\begin{subequations}
\begin{align}
\label{EqnDefnBivIfun}
\Ifun_i(\xi, \eta) \defn \Info\big( \Xvar_{\newrevinv(\xi), i};
\Xvar_{\newrevinv(\eta), -i} \big) \; = \;
\TC\big(\Xvar_{\newrevinv(\xi), i}, \Xvar_{\newrevinv(\eta), -i}\big)
- \TC\big(\Xvar_{\newrevinv(\eta), -i}\big).
\end{align}
Here $\partial_1$ and $\partial_2$ denote differentiation in the first
and second arguments of $\Ifun_i$, respectively.  Taking the mixed
partial derivative across the two arguments and then aggregating these
derivatives across coordinates yields the \emph{bivariate \igc kernel}
\begin{align}
  \label{EqnDefnBivQdens}
  \Bivq(\eta, \xi) & \defn \sum_{i = 1}^{\usedim} \partial_1
  \partial_2 \Ifun_i(\xi, \eta).
\end{align}
\end{subequations}
When evaluated on the diagonal, it turns out that we have the equivalence
$\qdens(\eta) = \Bivq(\eta, \eta)$, so that it is intimately connected
to the univariate \igc density. See~\Cref{LemExactKL} for the proof
of this equivalence.

\subsubsection{Integral representations using $\Bivq$}
\label{SecBivRepresentations}

Naturally, a key quantity in the proof of~\Cref{ThmMaster} is the
one-step KL discretization error, given by
\begin{subequations}
\begin{align}
\label{EqnDefnExactKL}  
\DenKL(p, q) & \defn \Exs_{X_p}\left[ \KL\left( \Kexact{p}{q}(
  \mathord\cdot \mid X_p ) \,\middle\|\, \Ktilde{p}{q}( \mathord\cdot
  \mid X_p ) \right) \right].
\end{align}  \end{subequations}
We now show how $\Bivq$ is fundamental in that it simultaneously
provides integral representations of the KL one-step
error~\eqref{EqnDefnExactKL}, the \igc increment, and the (dual) total
correlation functions.
\mygraybox{
  \begin{proposition}
\label{PropBivIntegral}    
For any pair $0 < a < b < 1$, we can write the KL product-reference
error over the interval $[a,b]$ as
\begin{subequations}
  \begin{align}
    \label{EqnBivKL}
    \DenKL(a, b) & = \int_{\logodds(a)}^{\logodds(b)}
    \int_\eta^{\logodds(b)} \Bivq(\eta, \xi) \, d\xi \, d\eta.
  \end{align}
  Moreover, the associated \igc increment takes the form
  \begin{align}
    \label{EqnBivIGC}
    \Ginfo(a, b) & = \int_{\logodds(a)}^{\logodds(b)} \Bivq(\eta,
    \eta) \, d\eta.
  \end{align}
  Finally, the total correlation (\TC) and dual total correlation
  (\DTC) have the representations
  \begin{align}
    \label{EqnBivTC}
    \TC(\Zvar) \stackrel{(\TC)}{=} \int_{-\infty}^{\infty}
    \int_\eta^\infty \Bivq(\eta, \xi) \, d\xi \, d\eta, \qquad
    \mbox{and} \qquad \DTC(\Zvar) \stackrel{(\DTC)}{=}
    \int_{-\infty}^{\infty} \int_{-\infty}^\eta \Bivq(\eta, \xi) \,
    d\xi \, d\eta.
  \end{align}
\end{subequations}    
  \end{proposition}
}
\noindent We prove the exact KL representation~\eqref{EqnBivKL} and
the \igc representation~\eqref{EqnBivIGC} as part of the proof
of~\Cref{ThmMaster}; in particular, see~\Cref{LemExactKL}.  Finally,
we prove the two representations~\eqref{EqnBivTC}
in~\Cref{SecProofBivIntegral}.


\subsubsection{Exponential $\Bivq$-structure and its consequences}
\label{SecIGCDecay}

As part of the proof of~\Cref{LemExactKL}, we show that the univariate
\igc density $\qdens(\eta) = \Bivq(\eta, \eta)$ corresponds to the
diagonal of the bivariate \igc kernel.  This equivalence underlies the
\igc increment relation~\eqref{EqnBivIGC}.  We next isolate a key
structural property of the bivariate \igc kernel $\Bivq$ that controls
this difference, and thereby leads to upper bounds on the aggregate
\igc mass $\Ginfo(0,1)$ in terms of total correlation and dual total
correlation.

A key property of the log-\SR-odds parameterization is that the
bivariate $\igc$ kernel at an off-diagonal value $\Bivq(\eta, \xi)$
has a sandwich that is exponential in the difference $|\eta - \xi|$.
More precisely, we have:
\mygraybox{
\begin{lemma}[Exponential sandwich for bivariate \igc density]
\label{LemBivSandwich}
For every $\eta, \xi \in \real$, we have the sandwich relation
\begin{align}
  \label{EqnBivSandwich}
  \frac{1}{2} e^{-|\xi - \eta|} \Bivq(\eta, \eta) &
  \stackrel{(i)}{\leq} \Bivq(\eta, \xi) \; \stackrel{(ii)}{\leq} 2
  e^{|\xi - \eta|} \Bivq(\eta, \eta).
\end{align}
\end{lemma}
}
\noindent See~\Cref{SecProofLemBivSandwich} for the proof of this
claim, which is a key result in the paper.  Underlying the
comparison~\eqref{EqnBivSandwich} is a type of (reverse) data
processing inequality (DPI) that is satisfied by the relative entropy
production function (derivative of the KL divergence) for
distributions along the noising path; see~\Cref{LemDPI} for details.
For analyzing (dual) total correlation, only the lower inequality (i)
is relevant.  We make use of the upper inequality when analyzing a
data-dependent sandwich estimator in~\Cref{PropDataSingle}.

This lemma, when combined with the integral
representations~\eqref{EqnBivTC} from~\Cref{PropBivIntegral}, has an
immediate consequence:

\mygraybox{
  \begin{corollary}[Upper bounds via $\TC$ and $\DTC$]
\label{CorUpperTC}    
 The aggregate \igc complexity is upper bounded as
 \begin{align}
   \label{EqnUpperTC}
\Ginfo(0,1) & \leq 2 \; \min \{ \TC(Z), \; \DTC(Z) \}.
    \end{align}
  \end{corollary}
}
\begin{proof}
We give the proof here, since it is a straightforward consequence
of~\Cref{LemBivSandwich} and~\Cref{PropBivIntegral}.  On the upper
half-plane $\xi \geq \eta$, inequality~\eqref{EqnBivSandwich}(i) gives
$\Bivq(\eta, \xi) \geq \frac{1}{2} e^{-(\xi - \eta)} \Bivq(\eta,
\eta)$.  Consequently, representation~\eqref{EqnBivTC} yields
\begin{align*}
  \TC(\Zvar) & \geq \frac{1}{2} \int_{-\infty}^{\infty} \Bivq(\eta,
  \eta) \left\{ \int_\eta^\infty e^{-(\xi - \eta)} \, d\xi \right\}
  \, d\eta \; = \; \frac{1}{2} \int_{-\infty}^{\infty} \Bivq(\eta, \eta)
  \, d\eta = \frac{1}{2} \Ginfo(0,1).
\end{align*}
The same argument on the lower half-plane gives the $\DTC$ bound, and
combining the two bounds yields the claim.
\end{proof}


\section{Multi-block analysis and certification}
\label{SecMultiBlock}

Thus far, we have exploited~\Cref{ThmMaster} only in a simple way, via
an equi-spaced grid~\eqref{EqnSingleBlockGrid} in log-\SR-odds.  As in
our past papers~\cite{Wai_UGC,Wai_DGC}, its simple structure leads
naturally to bounds based on multi-block schedules, which tailor the
spacing of the stepsizes to the underlying structure of the \igc
density.  This section is devoted to exploration of such guarantees,
along with data-dependent ways of estimating the \igc increments that
define optimal schedules.  Notably, in~\Cref{PropDataSingle}, we
introduce a sandwich estimator for the \igc increments that is based
on a novel Bregman divergence that emerges naturally from our
analysis.  This provides a practical way of estimating both \igc
increments and KL error.


\subsection{A multi-block guarantee}
\label{SecMBOracle}

Throughout this section, we use $\odds(r) \defn r/(1-r)$ for the
\SR-odds ratio, and $\logodds(r) = \log(r/(1-r))$ for the log-odds
ratio.  For a given interval $[\newrinit, \newrfinal]$ with $0 <
\newrinit < \newrfinal < 1$, suppose that we partition it into a
collection of $\Btot$ blocks, say of the form
\begin{align*}
  \Partition & \defn \big\{ [\block_\bind, \block_{\bind + 1}] \mid
  \bind = 0, \ldots, \Btot - 1 \big\}, \qquad \block_0 = \newrinit,
  \quad \block_\Btot = \newrfinal.
\end{align*}
For each block, define its log-\SR-odds length and \igc mass by
$\Sinfo_\bind \defn \logodds(\block_{\bind + 1}) -
\logodds(\block_\bind)$ and $\Ginfo_\bind \defn \Ginfo(\block_\bind,
\block_{\bind + 1})$.  Suppose that a total of $N_\bind$
equal-log-odds steps are allocated to block $\bind$. Then an
application of~\Cref{ThmMaster} gives
\begin{subequations}
\begin{align}
  \label{EqnIGCMultiUpper}
  \KL(\Prob_Z \| \Prob_{\Zhat}) & \leq 2 \sum_{\bind = 0}^{\Btot - 1}
  \left[ \exp\left\{ \frac{\Sinfo_\bind}{N_\bind} \right\} - 1 \right]
  \Ginfo_\bind + \BOUNDARY.
\end{align}

In principle, for a given iteration budget $N$, one could solve the
integer programming problem of minimizing this upper bound subject to
the constraint $\sum_{\bind = 0}^{\Btot-1} N_\bind = N$.  Apart from
constant factor differences, this turns out to be unnecessary.  In
particular, we now state a bound in terms of the \emph{\igc partition
complexity} given by
\begin{align}
  \label{EqnIGCPartComplex}
  \PartComp(\Partition) & \defn \left( \sum_{\bind = 0}^{\Btot - 1}
  \sqrt{\Sinfo_\bind \Ginfo_\bind} \right)^2.
\end{align}
We analyze the \prd scheme that traverses block $\bind$ using an
equi-spaced grid for which $1 + \rho_\bind$ is the maximum odds
multiplier for a step.  Equivalently, the log-odds spacing is at most
$\log(1 + \rho_\bind)$, where we define
\begin{align}
  \label{EqnIGCTrackMultiRho}
  \rho_\bind & \defn \min\left\{ 1,\, \frac{4
    \sqrt{\PartComp(\Partition)}}{N}
  \sqrt{\frac{\Sinfo_\bind}{\Ginfo_\bind}} \right\}, \qquad \bind = 0,
  \ldots, \Btot - 1.
\end{align}
\end{subequations}

\mygraybox{
\begin{proposition}[Constant-factor optimal guarantees for a $\Btot$-block partition]
\label{PropIGCTrackMulti}
Given an iteration budget $N \geq 2 \left\{ \Btot + 2 \big(
\logodds(\newrfinal) - \logodds(\newrinit) \big) \right\}$, the
$\Btot$-block sampler over $[\newrinit, \newrfinal]$ with maximum
within-block odds multipliers $1 + \rho_\bind$ as defined in
equation~\eqref{EqnIGCTrackMultiRho} uses at most $N$ rounds, and returns a sample
$\Zhat$ such that
\begin{align}
  \label{EqnIGCTrackMultiKL}
  \KL(\Prob_Z \| \Prob_{\Zhat}) & \leq \frac{8
    \PartComp(\Partition)}{N} + \BOUNDARY.
\end{align}
This explicit allocation is within a factor of four of the best
bound~\eqref{EqnIGCMultiUpper} for the given partition.
\end{proposition}
}

\noindent The proof follows \emph{mutatis mutandis} from our previous
paper~\cite{Wai_UGC}, with suitable replacements of the relevant
complexities.  The only change is that the factor two in
\Cref{ThmMaster} propagates to the bound~\eqref{EqnIGCTrackMultiKL}.

We note that partitioning can only improve the coarse complexity.
Indeed, by combining the additivity of the \igc function $\Ginfo$ with
the Cauchy--Schwarz inequality, we are guaranteed the ordering relation
\begin{align}
  \label{EqnIGCCauchy}  
  \PartComp(\Partition') & \leq \PartComp(\Partition) \; \leq \;
  \big\{ \logodds(\newrfinal) - \logodds(\newrinit) \big\}
  \Ginfo(\newrinit, \newrfinal),
\end{align}
for any pair $\Partition' \subset \Partition$ of nested partitions.
For the full symmetric interval $\IntStar$, the right-hand side is
exactly the single block complexity $\SB$ from
equation~\eqref{EqnDefnSingBlock}.  Strict improvement occurs whenever
the blockwise \igc mass is not proportional to log-odds length.

In contrast to~\Cref{CorSingleBlock}, for which the stepsize choice
required no knowledge of data-dependent parameters,
the odds-multiplier parameters~\eqref{EqnIGCTrackMultiRho} that underlie
\Cref{PropIGCTrackMulti} are based on the \igc increments $\Ginfo_\bind$
for each block.  Accordingly, in order to harness the power of the
result, we need to develop a method, assuming access to some collection
of samples from $\ProbZ$, for reliably estimating these increments.


\subsection{Bregman-based estimation of \igc increments and KL error}
\label{SecBregman}

We now turn to the key question of how to estimate the \igc
increments.  It turns out that there is a \emph{hidden Bregman
divergence}, one that is revealed by the proof of~\Cref{ThmMaster},
that can be used both to give a sandwich relation for the \igc
increments, and an exact relation for the KL error~\eqref{EqnExactKL}.

For each coordinate $i = 1, \ldots, d$, any probability vector $p$ on
$\Alphabet$ and reliability $s \in [0,1]$, we define the path
\begin{subequations}
\begin{align}
  \Tmap_{i,s}(p) & \defn s p + (1 - s) \refdist_i.
\end{align}
Note that this operator transforms $p$ along the path, parameterized
by the reliability $s$.  For each coordinate $i \in [\usedim]$ and
squared reliability $r \in (0, 1)$, we use $\DenLoo_{r,i}(\cdot,
\Xvar_{r,-i})$ to denote the random leave-one-out posterior probability
vector.

Fix $0 < a < b < 1$, and for any anchor $u \in [a, b]$, we define the
\emph{anchored \igc increment}
\begin{align}
  \label{EqnDefnAnchorIGC}
  \Dinfo_u(a, b) & \defn u (1 - u) \sum_{i = 1}^{\usedim} \Exs\left[
    \frac{d}{d u} \KL\left( \Tmap_{i,\sqrt{u}}\big(
    \DenLoo_{b,i}(\cdot, \Xvar_{b,-i}) \big) \,\middle\|\,
    \Tmap_{i,\sqrt{u}}\big( \DenLoo_{a,i}(\cdot, \Xvar_{a,-i}) \big)
    \right) \right].
\end{align}
\end{subequations}
Although this function \emph{appears} to require computing derivatives
of the KL along the path, it turns out that this KL derivative is
itself equal to a different Bregman divergence.  At a high level, this
correspondence arises because the map $\Tmap_{i,s}$ is linear in the
parameter; the KL divergence is itself a Bregman divergence; and
Bregman divergences have a linear representation in terms of their
generators.  For full details, see equation~\eqref{EqnKeyBregman} in the
proof of~\Cref{LemDPI}, and the explicit computation of the Bregman
generator in equation~\eqref{EqnBregmanGen}.  This explicit Bregman
representation is critical, because it means that we can use samples
to obtain an explicit unbiased estimate of $\Dinfo_u(a,b)$.

The anchored \igc increment turns out to have some special properties,
ones which allow us to relate it to both the exact KL error and the
\igc increment for any interval $[a,b]$.  More precisely, for any
triple $0 < a < u < b < 1$, let us introduce the shorthand $c_u(a, b)
\defn \max \big\{ \frac{\odds(u)}{\odds(a)}, \frac{\odds(b)}{\odds(u)}
\big\}$, where $\odds(r) = r/(1-r)$ is the odds ratio.

\mygraybox{
\begin{proposition}[Bregman-based estimation of KL error and \igc increments]
\label{PropDataSingle}
For any choice of $u \in [a,b]$, the \igc increment is sandwiched as
\begin{subequations}
\begin{align}
  \label{EqnDinfoIncrementEstimate}
  \frac{1}{2 c_u(a, b)} \Dinfo_u(a, b) & \leq \Ginfo(a, b) \leq 2
  c_u(a, b) \Dinfo_u(a, b).
\end{align}
Moreover, the one-step KL error~\eqref{EqnExactKL} is equal to
\begin{align}
\label{EqnDinfoKLRepresentation}
\DenKL(a, b) & = \int_a^b \frac{\Dinfo_v(a, v)}{v (1 - v)} \, dv.
\end{align}
\end{subequations}
\end{proposition}
}
\noindent See~\Cref{SecProofPropDataSingle} for the proof.

In terms of the sandwich
relation~\eqref{EqnDinfoIncrementEstimate}, a particularly convenient
choice of $u$ is the log-odds midpoint $u_\star$, defined by the
relation $\logodds(u_\star) = \frac{\logodds(a) + \logodds(b)}{2}$.
This choice ensures that
\begin{align}
  \label{EqnDinfoMidpointEstimate}
  \frac{1}{2} \sqrt{\frac{\odds(a)}{\odds(b)}} \Dinfo_{u_\star}(a, b)
  & \leq \Ginfo(a, b) \leq 2 \sqrt{\frac{\odds(b)}{\odds(a)}}
  \Dinfo_{u_\star}(a, b),
\end{align}
so that we have the natural factor-two sandwich.

In summary, \Cref{PropDataSingle} provides a population-level
characterization of both the KL error and the \igc increment as an
expectation involving the \igc anchor functions $\Dinfo_u$.  Thus, if
we are given a collection of clean samples
$\{\Zsam{\ell}\}_{\ell=1}^m$ from $\ProbZ$, then we can use them to
generate a collection of noisy trajectories on a discretized grid of
squared-reliability times.  Plugging these samples into the
representation~\eqref{EqnDefnAnchorIGC} then yields an unbiased Monte
Carlo estimate.

In our previous work~\cite{Wai_UGC,Wai_DGC}, we proposed and analyzed
a tail-robust estimator, based on only a moment bound and a form of
truncation.  We can adapt this estimator to the current setting, and
use it to guarantee that, for a given sample size $m$, we can generate
data-dependent quantities $\Ghat_\bind$ and an upper confidence
function $\rhat_{\bind, m}(\delta)$ such that, for any $\delta \in
(0, 1)$, we have
\begin{align}
\label{EqnDefnGhat}  
  \Ginfo_\bind & \leq \Ghat_\bind + \rhat_{\bind, m}(\delta) \qquad
  \mbox{with probability at least $1 - \delta$.}
\end{align}
Note that the confidence bound $\rhat_{\bind, m}$ depends on the block
$\bind$ and the number $m$ of clean samples $\Zvar$, and under mild
moment conditions, for a fixed $\delta \in (0,1)$, we have
$\rhat_{\bind, m}(\delta) = \order(m^{-1/2})$ as $m$ grows.  This
guarantee is a direct application of our previous results, so that we
omit the details here.


\subsection{Data-certified multi-block guarantee}
\label{SecDataCert}

We are now set up to combine~\Cref{PropIGCTrackMulti}
with~\Cref{PropDataSingle} so as to obtain an end-to-end guarantee for
a multi-block \prd sampler.  Fix a partition $\Partition$ and a
failure probability $\eta \in (0, 1)$.  Compute the estimator with
failure probability $\delta \defn \eta/\Btot$ for each block.  We are
thus given data-dependent quantities $\Ghat_\bind$ and tolerances
$\rhat_\bind \equiv \rhat_{\bind,m}(\eta/\Btot)$ such that
\begin{subequations}
\begin{align}
  \label{EqnIGCConfidence}
  \Ginfo_\bind \leq \Ghat_\bind + \rhat_\bind \quad \mbox{for all
    $\bind = 0, \ldots, \Btot - 1$, with probability at least $1 -
    \eta$.}
\end{align}
Define the certified partition complexity
\begin{align}
  \PartCompHat(\Partition) & \defn \left(
    \sum_{\bind = 0}^{\Btot - 1}
    \sqrt{\Sinfo_\bind \big( \Ghat_\bind + \rhat_\bind \big)}
  \right)^2.
\end{align}
The certified odds-multiplier parameters are
\begin{align}
  \label{EqnIGCCertifiedRho}
  \rhohat_\bind & \defn \min\left\{ 1,\, \frac{4
    \sqrt{\PartCompHat(\Partition)}}{N} \sqrt{\frac{\Sinfo_\bind}{
      \Ghat_\bind + \rhat_\bind}} \right\}, \qquad \bind = 0, \ldots,
  \Btot - 1,
\end{align}
\end{subequations}
Thus, within block $\bind$, the maximum certified odds multiplier is
$1 + \rhohat_\bind$, equivalently the log-odds spacing is at most
$\log(1 + \rhohat_\bind)$.
\mygraybox{
\begin{theorem}[Data-certified multi-block guarantee]
  \label{ThmIGCCertifiedMulti}
Given an iteration budget $N \geq 2 \big \{ \Btot + 2 \big(
\logodds(\newrfinal) - \logodds(\newrinit) \big) \big\}$ and the
simultaneous confidence event~\eqref{EqnIGCConfidence}, the
$\Btot$-block sampler with maximum within-block odds multipliers $1 +
\rhohat_\bind$ determined by the parameters~\eqref{EqnIGCCertifiedRho}
has output $\Zhat$ such that
\begin{align}
  \label{EqnIGCCertifiedKL}
  \KL(\Prob_Z \| \Prob_{\Zhat}) & \leq \frac{8
    \PartCompHat(\Partition)}{N} + \BOUNDARY,
\end{align}
with probability at least $1 - \eta$.
\end{theorem}
}
\noindent The proof follows by the same argument as
in~\cite{Wai_UGC}, with suitable replacements.


\subsection{The \igc density and optimal Euler discretizations}
\label{SecFine}

Finally, we can pass from finitely many blocks to the fine-partition
limit.  From the partition-refinement inequality~\eqref{EqnIGCCauchy},
by a sequence of progressively refined $\Btot$-block partitions, we
can approach the \emph{fine-partition limit}
\begin{align}
  \label{EqnDefnIGCFinePart}
  \PartHfine(\IntStar) & \defn \inf_{\Partition}
  \PartComp(\Partition),
\end{align}
where the infimum is over all finite partitions of $\IntStar$.  The
next result identifies this quantity with the square-root integral
introduced in \eqref{EqnDefnFinePart}, and shows that the same
quantity governs the sharp leading-order Euler error.

\mygraybox{
\begin{theorem}[Fine-partition limit and Euler optimality]
\label{ThmKL}
Assume that the mixed partial derivatives appearing
in~\Cref{LemExactKL} are continuous on the compact log-odds region
associated with $\IntStar$, and that $\qdens$ is strictly positive on
$[-\Elld, \Elld]$.  Then we have the equivalence
\begin{subequations}
\begin{align}
  \label{EqnIGCFinePartIdentity}
  \PartHfine(\IntStar) & = \left( \int_{-\Elld}^{\Elld}
  \sqrt{\qdens(\lam)} \, d\lam \right)^2 \equiv \FP.
\end{align}
Moreover, for the log-odds-uniform grid~\eqref{EqnSingleBlockGrid},
the sum of exact KL one-step errors satisfies
\begin{align}
  \label{EqnIGCUniformEuler}
  \sum_{j = 0}^{N - 1} \DenKL(r_j, r_{j + 1}) & = \frac{\SB}{2 N} +
  o(N^{-1}),
\end{align}
Finally, the optimal $N$-step discretization satisfies
\begin{align}
  \label{EqnIGCFineEulerOptimality}
  \inf_{\rstar = r_0 < r_1 < \cdots < r_N = 1 - \rstar} \sum_{j =
    0}^{N - 1} \DenKL(r_j, r_{j + 1}) & = \frac{\FP}{2 N} + o(N^{-1}).
\end{align}
\end{subequations}
\end{theorem}
}
\noindent
The proof is the same fine-partition argument as in our past work on
masking~\cite{Wai_UGC}.



\section{Proofs}
\label{SecProofs}

We now turn to the proofs of our main results.  The proof
of~\Cref{ThmMaster} is given in~\Cref{SecProofThmMaster}, whereas
\Cref{SecProofPropDataSingle} contains the proof
of~\Cref{PropDataSingle}.

\subsection{Proof of~\Cref{ThmMaster}}
\label{SecProofThmMaster}

Recall that $\{\Xvar_t, t \in [0,1] \}$ corresponds to the denoising
process~\eqref{EqnGeneralRefTransition}, which transitions from the
product-reference distribution $\Xvar_0 \sim \bigotimes_{i =
  1}^{\usedim} \refdist_i$ to $\Xvar_1 \sim \ProbZ$.

For any pair $0 < p < q < 1$, let $\Kexact{p}{q}$ denote the
transition kernel of this denoising process in moving from $X_p$ to
$X_q$.  On the other hand, our denoising sampler is based on the
update~\eqref{EqnPRDUpdate} and we let $\Ktilde{p}{q}$ denote the
associated transition kernel that moves from $\Xhat_p$ to $\Xhat_q$.
We introduce the shorthand notation
\begin{align}
\label{EqnDefnUnmaskKL}  
\DenKL(p, q) & \defn \Exs_{X_p}\left[ \KL\left( \Kexact{p}{q}(
  \mathord\cdot \mid X_p ) \,\middle\|\, \Ktilde{p}{q}( \mathord\cdot
  \mid X_p ) \right) \right].
\end{align}
corresponding to the averaged Kullback--Leibler (KL) discrepancy
between the two transition kernels over the interval $[p,q]$.  Here
the expectation is taken over the marginal distribution of the
variable $X_p$ along the denoising path.

By the data-processing inequality combined with the chain rule for KL
divergence, we have
\begin{align}
\label{EqnDataProcessing}
  \KL( \Prob_Z \| \Prob_{\Zhat}) & \leq \sum_{j = 0}^{N - 1}
  \DenKL(\time_j, \time_{j + 1}) + \BOUNDARY.
\end{align}
Thus, the central challenge is to bound the one-step KL errors
$\DenKL(\time_j, \time_{j + 1})$.  The following result provides the requisite
control.
\mygraybox{
\begin{lemma}
\label{LemGeoOneStep}
For any $0 < a < b < 1$, the one-step KL error satisfies the bound
\begin{align}
  \label{EqnGeoOneStep}
  \DenKL(a, b) & \leq 2 \left\{ \frac{\odds(b)}{\odds(a)} - 1 \right\}
  \Ginfo(a, b),
\end{align}
where $\odds(r) \defn r/(1 - r)$.
\end{lemma}
}
\noindent See~\Cref{SecProofLemGeoOneStep} for the proof of this
claim.

To complete the proof of~\Cref{ThmMaster}, we
apply~\Cref{LemGeoOneStep} to each of the intervals $[\time_j, \time_{j + 1}]$.
Collecting the terms yields the claim~\eqref{EqnMaster}.


\subsubsection{Proof of~\Cref{LemGeoOneStep}}
\label{SecProofLemGeoOneStep}

This proof is based on two lemmas, each of which is of independent
interest in its own right.  The first is the previously
stated~\Cref{LemBivSandwich}, which guarantees that
\begin{align}
  \label{EqnBivSandwichTwo}
  \frac{1}{2} e^{-|\xi - \eta|} \Bivq(\eta, \eta) &
  \stackrel{(i)}{\leq} \Bivq(\eta, \xi) \; \stackrel{(ii)}{\leq} 2
  e^{|\xi - \eta|} \Bivq(\eta, \eta).
\end{align}
We used the lower bound (i) previously in comparing the coarse \igc
mass to the (dual) total correlations.  In this proof, the relevant
result is the upper bound (ii).  See~\Cref{SecProofLemBivSandwich} for
the proof of~\Cref{LemBivSandwich}.

The second piece in the proof of~\Cref{LemGeoOneStep} is an exact
representation of the one-step KL error in terms of the bivariate \igc
kernel $\Bivq$.  This result is also of independent interest.
\mygraybox{
  \begin{lemma}[Exact representation of the one-step KL divergence]
\label{LemExactKL}
For any pair $0 < a < b < 1$, the one-step KL
divergence~\eqref{EqnDefnExactKL} has the exact representation
\begin{subequations}
  \begin{align}
\label{EqnExactKL}    
\DenKL(a, b) & = \int_{\lam(a)}^{\lam(b)} \int_\eta^{\lam(b)}
\Bivq(\eta, \xi) \, d\xi \, d\eta \qquad \mbox{where $\lam(r) =
  \log(r/(1-r))$.}
\end{align}
Moreover, assuming that $\Bivq$ is continuous, we have the pointwise
expansion
\begin{align}
\label{EqnPointwiseKL}
\DenKL(a, b) & = \frac{1}{2} \Bivq(\lam(a), \lam(a)) \; \Delta^2 +
o(\Delta^2) \qquad \mbox{ where $\Delta \defn \lam(b) - \lam(a)$.}
\end{align}
Finally, we have the equivalence
\begin{align}
\label{EqnQdensDiagonal}  
\qdens(\lam) & = \Bivq(\lam, \lam) \qquad \mbox{for all $\lam \in
  \real$.}
\end{align}
\end{subequations}
\end{lemma}
}
\noindent See~\Cref{SecProofLemExactKL} for the proof.

We note that~\Cref{LemGeoOneStep} is a straightforward consequence of
the upper inequality (ii) in equation~\eqref{EqnBivSandwichTwo}, and
the KL representation in~\Cref{LemExactKL}.  In particular, we 
have
\begin{align*}
  \DenKL(a, b) \; \stackrel{(i)}{=} \; \int_{\lam(a)}^{\lam(b)}
  \int_\eta^{\lam(b)} \Bivq(\eta, \xi) \, d\xi \, d\eta &
  \stackrel{(ii)}{\leq} 2 \int_{\lam(a)}^{\lam(b)} \int_\eta^{\lam(b)}
  e^{\xi - \eta} \Bivq(\eta, \eta) \, d\xi \, d\eta \\
  & = 2 \int_{\lam(a)}^{\lam(b)} \big\{ e^{\lam(b) - \eta} - 1 \big\}
  \Bivq(\eta, \eta) \, d\eta \\
  & \stackrel{(iii)}{\leq} 2 \big\{ e^{\lam(b) - \lam(a)} - 1 \big\}
  \int_{\lam(a)}^{\lam(b)} \Bivq(\eta, \eta) \, d\eta \\
  & = 2 \left\{ \frac{\odds(b)}{\odds(a)} - 1 \right\} \Ginfo(a, b),
\end{align*}
where step (i) uses~\Cref{LemExactKL}; step (ii) uses
equation~\eqref{EqnBivSandwichTwo}; and step (iii) follows from the
lower bound $\eta \geq \lam(a)$ and the non-negativity of $\Bivq(\eta,
\eta)$.


\subsubsection{Proof of~\Cref{LemExactKL}}
\label{SecProofLemExactKL}

We split our proof into three parts, one for each claim.

\paragraph{Proof of the exact representation~\eqref{EqnExactKL}:}
Introducing the shorthand $\alpha \defn \lam(a)$ and $\beta \defn
\lam(b)$, we introduce the functions $T(\xi) \defn
\TC(\Xvar_{r(\xi)})$ and $A_i(\xi) \defn \TC\big( \Xvar_{r(\xi),i},
\Xvar_{a,-i} \big)$.  By the definition~\eqref{EqnDefnBivIfun} of
$\Ifun_i$, we have $A_i'(\xi) = \partial_1 \Ifun_i(\xi, \alpha)$, and
we prove in~\Cref{SecProofBivIntegral} that $T'(\xi) = \sum_{i =
  1}^{\usedim} \partial_1 \Ifun_i(\xi,\xi)$.

To verify the decomposition, writing $U \defn \Xvar_b$ and $V \defn
\Xvar_a$, coordinatewise independence of the channel gives
\begin{align*}
  \TC(U \mid V) = \sum_{i = 1}^{\usedim} \Ent(U_i \mid V) - \Ent(U
  \mid V) & = \sum_{i = 1}^{\usedim} \Ent(U_i, V_{-i}) - (\usedim - 1)
  \Ent(V) - \Ent(U) \\ & = \big\{ \TC(U) - \TC(V) \big\} - \sum_{i =
    1}^{\usedim} \big\{ \TC(U_i, V_{-i}) - \TC(V) \big\},
\end{align*}
where the second equality follows by cancellation of the terms
$\sum_{i = 1}^{\usedim} \Ent(V_i \mid U_i)$.  Combined with the
identity $\DenKL(a, b) = \TC(\Xvar_b \mid \Xvar_a)$, substituting the
definitions of $T$ and $A_i$ yields
\begin{align*}
  \DenKL(a, b) = \big\{ T(\beta) - T(\alpha) \big\} - \sum_{i =
    1}^{\usedim} \big\{ A_i(\beta) - A_i(\alpha) \big\}.
\end{align*}

Combining the ingredients, we find that
\begin{align*}
  \DenKL(a, b) = \sum_{i = 1}^{\usedim} \int_\alpha^\beta \left\{
  \partial_1 \Ifun_i(\xi, \xi) - \partial_1 \Ifun_i(\xi, \alpha)
  \right\} \, d\xi & \stackrel{(i)}{=} \sum_{i = 1}^{\usedim}
  \int_\alpha^\beta \int_\alpha^\xi \partial_1 \partial_2
  \Ifun_i(\xi,\eta) \, d\eta \, d\xi \\ & \stackrel{(ii)}{=}
  \int_\alpha^\beta \int_\alpha^\xi \Bivq(\eta, \xi) \, d\eta \, d\xi
  \\ & \stackrel{(iii)}{=} \int_\alpha^\beta \int_\eta^\beta
  \Bivq(\eta, \xi) \, d\xi \, d\eta,
\end{align*}
where step (i) follows from the fundamental theorem of calculus; step
(ii) follows from the definition of $\Bivq$; and step (iii) follows by
reversing the order of integration over the triangle $\alpha \leq \eta
\leq \xi \leq \beta$.

\paragraph{Proof of the pointwise expansion~\eqref{EqnPointwiseKL}:}

Introduce the shorthand $\alpha \defn \lam(a)$ and $\Delta \defn
\lam(b) - \lam(a)$.  By the continuity of $\Bivq$ at
$(\alpha,\alpha)$, we have $\sup_{\alpha \leq \eta \leq \xi \leq
  \alpha + \Delta} \left| \Bivq(\eta, \xi) - \Bivq(\alpha, \alpha) \right|
= o(1)$ as $\Delta \downarrow 0$.  Since the integration triangle has
area $\Delta^2/2$, it follows that
\begin{align*}
  \DenKL(a, b) & = \Bivq(\alpha, \alpha) \int_\alpha^{\alpha + \Delta}
  \int_\eta^{\alpha + \Delta} \, d\xi \, d\eta + o(\Delta^2) \; = \;
  \frac{1}{2} \Bivq(\alpha, \alpha) \Delta^2 + o(\Delta^2) \; = \;
  \frac{1}{2} \qdens(\lam(a)) \Delta^2 + o(\Delta^2),
\end{align*}
which proves the claim.

\paragraph{Proof of the equivalence~\eqref{EqnQdensDiagonal}:}

Define the mutual-information surface $F(\blam) \defn \Info\big(
\Zvar; \Xvar_{\rbold(\blam)} \big)$, where $r(\lam_k) \defn
\frac{e^{\lam_k}}{1 + e^{\lam_k}}$.  In terms of this function, we have
\begin{align}
  \label{EqnQdensLambdaHessian}
  \qdens(\lam) & = -\sum_{i \neq j} \left.  \frac{\partial^2 F(\blam)}
        {\partial \lam_i \partial \lam_j} \right|_{\blam = \lam
          \onevec}.
\end{align}
Now fix a coordinate $i$.  Since the coordinatewise refresh channels
are conditionally independent given $\Zvar$, we have
\begin{align*}
  \Ifun_i(\xi,\eta) & = \Info\big( \Zvar; \Xvar_{r(\xi),i} \big) +
  \Info\big( \Zvar; \Xvar_{r(\eta),-i} \big) - F\big( \xi \evec_i +
  \eta(\onevec-\evec_i) \big).
\end{align*}
The first two terms depend on only one of $(\xi,\eta)$, and therefore
\begin{align*}
  \partial_1 \partial_2 \Ifun_i(\xi,\eta) & = -\sum_{j \neq i} \left.
  \frac{\partial^2 F(\blam)} {\partial \lam_i \partial \lam_j}
  \right|_{\blam = \xi \evec_i + \eta(\onevec-\evec_i)}.
\end{align*}
Setting $\xi=\eta=\lam$ and summing over $i$ now gives
\begin{align*}
  \Bivq(\lam,\lam) & = \sum_{i = 1}^{\usedim} \partial_1 \partial_2
  \Ifun_i(\lam,\lam) \; = \; - \sum_{i \neq j} \left.
  \frac{\partial^2 F(\blam)} {\partial \lam_i \partial \lam_j}
  \right|_{\blam = \lam \onevec} = \qdens(\lam),
\end{align*}
where the last equality follows from
equation~\eqref{EqnQdensLambdaHessian}.

\subsubsection{Proof of~\Cref{LemBivSandwich}}
\label{SecProofLemBivSandwich}

Our proof makes use of an auxiliary result, one which exploits the
geometry of our noising path in an essential way.  Recalling that
$r(\lam) = \frac{e^{\lam}}{1 + e^{\lam}}$, we let $s_\lambda =
\sqrt{r(\lambda)}$ denote the ordinary reliability at level $\lambda$.
For any reference distribution $\refdist$ over $\Alphabet$, we define
the path
\begin{subequations}
    \begin{align}
\label{EqnPath}      
  \Pmap_\lam(p) & = s_\lambda p + (1 - s_\lambda) \refdist, \qquad
  \mbox{for $s_\lambda \in [0,1]$.}
\end{align}
This path describes the evolution from the reference distribution at time
$s_\lam = 0$ to its argument $p$ at time $s_\lam = 1$.  We then define
the $\lambda$-indexed family of KL divergences
\begin{align}
\label{EqnLamKL}
  \mathsf K_\lambda(p, q) & \defn \KL\left( \Pmap_\lam(p) \|
  \Pmap_\lam(q) \right),
\end{align}
as well as its derivative
\begin{align}
\label{EqnLamJfun}
\mathsf J_\lambda(p, q) & \defn \frac{d}{d \lambda} \mathsf
K_\lambda(p, q).
\end{align}
\end{subequations}
By the data-processing inequality, we have $\mathsf K_\alpha(p, q)
\leq \mathsf K_\beta(p, q)$ whenever $\alpha < \beta$, from which it
follows that $\mathsf J_\lam(p, q) \geq 0$ for all $\lambda \in \real$.

\mygraybox{
\begin{lemma}[Exponential control along the path]
  \label{LemDPI}
  Fix an arbitrary pair $p, q$ of probability distributions over
  $\Alphabet$.  Then for all $\alpha, \beta \in \real$, we have the
  sandwich relation
\begin{align}
  \label{EqnDPISandwich}
  \frac{1}{2} e^{-|\beta - \alpha|} \mathsf J_\alpha(p, q) & \leq
  \mathsf J_\beta(p, q) \leq 2 e^{|\beta - \alpha|} \mathsf
  J_\alpha(p, q).
\end{align}
\end{lemma}
}
\noindent See~\Cref{SecProofLemDPI} for the proof of this claim.

We now use~\Cref{LemDPI} to complete the proof
of~\Cref{LemBivSandwich}.  Fix a coordinate $i$ and log-\SR-odds times
$\eta \in \real$ and $\eta + h$, where $h > 0$.  Write $\pi_{i,\zeta}
\defn \Law\big( \Zvar_i \mid \Xvar_{\newrevinv(\zeta),-i} \big)$ for
the leave-one-out posterior at log-\SR-odds level $\zeta$.  Let $s_\xi
\defn \sqrt{r(\xi)}$ be the ordinary reliability time associated with
$\xi$.

Since $h > 0$, our (backward) noising process generates
$\Xvar_{\newrevinv(\eta), -i}$ by adding additional independent noise
to $\Xvar_{\newrevinv(\eta+h), -i}$, meaning that the triple
$\Xvar_{\newrevinv(\xi),i} \longrightarrow
\Xvar_{\newrevinv(\eta+h),-i} \longrightarrow
\Xvar_{\newrevinv(\eta),-i}$ forms a Markov chain.  Consequently, we
have
\begin{subequations}
\begin{align}
  \label{EqnPostIncrement}
  \Ifun_i(\xi,\eta+h) - \Ifun_i(\xi,\eta) & \stackrel{(i)}{=}
  \Info\big( \Xvar_{\newrevinv(\xi),i}; \Xvar_{\newrevinv(\eta+h),-i}
  \mid \Xvar_{\newrevinv(\eta),-i} \big) \; \stackrel{(ii)}{=}\;
  \Exs\left[ \mathsf K_\xi\big( \pi_{i,\eta+h}, \pi_{i,\eta} \big)
    \right],
\end{align}
where step (i) uses the chain rule for mutual information; and step
(ii) uses the posterior tower property, the representation of
conditional mutual information as an expected KL, together with the
definition~\eqref{EqnLamKL}.

Differentiating equation~\eqref{EqnPostIncrement} in its first
argument yields $\partial_1 \Ifun_i(\xi, \eta + h) - \partial_1
\Ifun_i(\xi, \eta) = \Exs\left[ \mathsf J_\xi\big( \pi_{i,\eta + h},
  \pi_{i,\eta} \big) \right]$, and hence
  \begin{align}
    \label{EqnMixed}
    \partial_1 \partial_2 \Ifun_i(\xi, \eta) & = \lim_{h \downarrow 0}
    \frac{1}{h} \Exs\left[ \mathsf J_\xi\big( \pi_{i,\eta + h},
      \pi_{i,\eta} \big) \right].
  \end{align}
\end{subequations}
The sandwich bound~\eqref{EqnDPISandwich} from~\Cref{LemDPI}
guarantees that
\begin{align*}
  \frac{1}{2} e^{-|\xi - \eta|} \mathsf J_\eta\big( \pi_{i,\eta + h},
  \pi_{i,\eta} \big) & \leq \mathsf J_\xi\big( \pi_{i,\eta + h},
  \pi_{i,\eta} \big) \leq 2 e^{|\xi - \eta|} \mathsf J_\eta\big(
  \pi_{i,\eta + h}, \pi_{i,\eta} \big).
\end{align*}
We now take expectations, divide by $h > 0$, and pass to the limit.
Combined with equation~\eqref{EqnMixed}, we see that
\begin{align*}
  \frac{1}{2} e^{-|\xi - \eta|} \partial_1 \partial_2 \Ifun_i(\eta,
\eta) & \leq \partial_1 \partial_2 \Ifun_i(\xi, \eta) \leq 2 e^{|\xi -
  \eta|} \partial_1 \partial_2 \Ifun_i(\eta, \eta).
\end{align*}
Summing over coordinates $i$ yields the claim~\eqref{EqnBivSandwich}.

\subsubsection{Proof of~\Cref{LemDPI}}
\label{SecProofLemDPI}
Suppose first that $\refdist$ has full support.  Recall that the
function $\phi(x) \defn x \log x$ generates the KL divergence as the
Bregman divergence $\Breg_\phi$ given by
\begin{align*}
  \KL(p \| q) & \equiv \Breg_\phi(p \| q) \defn \sum_{a \in \Alphabet}
  \big \{ \phi(p_a) - \phi(q_a) - \phi'(q_a) \, (p_a - q_a) \big \}.
\end{align*}
Recall that $\mathsf K_\lam(p \| q) \defn \KL(\Pmap_\lam(p) \|
\Pmap_\lam(q))$ by definition, as well as the
definition~\eqref{EqnPath} of the path $\Pmap_\lam$.

Since $\Pmap_\lam$ is an affine function, it follows that $\mathsf
K_\lam(p \| q)$ can be seen as a family of Bregman divergences: in
particular, we have
\begin{align*}
  \mathsf K_\lam(p \| q) & = \Breg_{\Phi_\lambda}(p \| q) \qquad
  \mbox{where $\Phi_\lambda(z) = \sum \limits_{a \in \Alphabet} \phi
    \big( [\Pmap_\lambda(z)]_a\big)$.}
\end{align*}
Since the definition of a Bregman divergence is linear in
$\Phi_\lambda$, it follows that
\begin{align}
  \label{EqnKeyBregman}
\frac{d}{d\lambda} \Breg_{\Phi_\lambda}(p \| q) \; = \; \mathsf J_\lambda(p,
q) \; = \Breg_{\Psi_\lambda}(p \| q), \qquad \mbox{where $
  \Psi_\lambda(z) \defn \frac{d}{d \lambda} \Phi_\lambda(z)$.}
\end{align}
Consequently, by the integral form of a Bregman divergence, we can
write
\begin{align*}
  \mathsf J_\lambda(p, q) & = \int_0^1 (1 - t) (p - q)^{\mathsf T} \nabla^2
  \Psi_\lambda\big(q + t (p - q)\big) (p - q) \, dt.
\end{align*}
It suffices to consider $\alpha < \beta$: the case $\beta < \alpha$
follows by interchanging the two parameters and using the resulting
bounds in the reverse direction to obtain the symmetric bound in
equation~\eqref{EqnDPISandwich}, while $\alpha = \beta$ is immediate.
Consequently, in order to complete the proof, it suffices to show that
\begin{align*}
  \frac{1}{\ehack} \nabla^2\Psi_\alpha(z) & \preceq
  \nabla^2\Psi_\beta(z) \preceq 2 \ehack \nabla^2\Psi_\alpha(z) \qquad
  \mbox{where $\ehack \defn e^{\beta - \alpha}$ for an arbitrary pair
    $\alpha < \beta$.}
\end{align*}

We now compute $\Psi_\lambda$.  Observe that
$\frac{d}{d\lambda}s_\lambda = \frac{1}{2} s_\lambda(1 -
s_\lambda^2)$.  Thus, by the chain rule, together with the relation
$\sum_{a \in \Alphabet} (z_a - \refdist_a) = 0$, we
obtain
\begin{subequations}
\begin{align}
\label{EqnBregmanGen}  
  \Psi_\lambda(z) & = \frac{1}{2} s_\lambda(1 - s_\lambda^2)
  \sum_{a \in \Alphabet} (z_a - \refdist_a) \log
      [\Pmap_\lambda(z)]_a.
\end{align}
Since $\Psi_\lambda$ is additive across coordinates, its Hessian is
diagonal with
\begin{align}
  \label{EqnJHessian}
  [\nabla^2\Psi_\lambda(z)]_{aa} & = \frac{s_\lambda^2(1 -
    s_\lambda^2)}{2} \frac{ [\Pmap_\lambda(z)]_a + \refdist_a }{
    [\Pmap_\lambda(z)]_a^2 }.
\end{align}
\end{subequations}

For a fixed $\alpha < \beta$, introduce the shorthand $s \defn
s_\alpha$ and $t \defn s_\beta$.  Since $\lambda$ is the log-\SR-odds
parameter, we can compute $\ehack = e^{\beta - \alpha} = \frac{t^2(1 -
  s^2)}{s^2(1 - t^2)}$.  For a fixed symbol $a$, set $x \defn
z_a/\refdist_a$ and $y_u(x) \defn 1 + u(x - 1)$.  Then
$[\Pmap_\alpha(z)]_a = \refdist_a y_s(x)$ and $[\Pmap_\beta(z)]_a =
\refdist_a y_t(x)$, so that the Hessian from
equation~\eqref{EqnJHessian} gives the ratio
\begin{align*}
  R(x) & \defn \frac{ [\nabla^2\Psi_\beta(z)]_{aa} }{
    [\nabla^2\Psi_\alpha(z)]_{aa} } = \frac{t^2(1 - t^2)}{s^2(1 -
    s^2)} \frac{(y_t(x) + 1) y_s^2(x)}{y_t^2(x)(y_s(x) + 1)}.
\end{align*}
To complete the proof, it suffices to show that $R(x) \in \big[
  \frac{1}{\ehack}, 2 \ehack]$.

We can compute
\begin{align*}
  R'(x)
  & =
  \frac{t^2(1 - t^2)}{s^2(1 - s^2)}
  \frac{2 (s - t) y_s(x)
  \big( y_s(x) + y_t(x) + 1 \big)}
  {\big( y_s(x) + 1 \big)^2 y_t^3(x)}
  \leq 0.
\end{align*}
Since $s < t$, the derivative is nonpositive, so $R(x)$ must lie
between its endpoint values.  Note that $x =
z_a/\refdist_a$ takes
values in the interval $[0, \infty)$, and we have
\begin{align*}
  \lim_{x \to \infty} R(x) = \frac{1}{\ehack} \left( \frac{t}{s}
  \right)^3 \geq \frac{1}{\ehack} \quad \mbox{and} \quad R(0) = \ehack
  \frac{(2 - t)(1 + t)^2}{(2 - s)(1 + s)^2} \leq 2 \ehack,
\end{align*}
where the last inequality uses that the function $u \mapsto (2 - u)(1
+ u)^2$ increases from $2$ to $4$ over the interval $[0,1]$.

Finally, for an arbitrary $\refdist$, we can approximate it by a
sequence of full-support reference distributions and pass to the limit
(assuming that the displayed derivatives are finite).


\subsection{Proof of~\Cref{PropDataSingle}}
\label{SecProofPropDataSingle}

We introduce the shorthand notation $\alpha \defn \logodds(a)$, $\beta
\defn \logodds(b)$, and $\xi \defn \logodds(u)$.  By the definitions
of $\Dinfo_u(a, b)$ and $\Bivq$, we have
\begin{align*}
\Dinfo_u(a, b) = \int_\alpha^\beta \Bivq(\eta, \xi) \, d\eta, \quad
\mbox{and} \quad \Ginfo(a, b) = \int_\alpha^\beta \Bivq(\eta, \eta)
\, d\eta.
\end{align*}
For every $\eta \in [\alpha, \beta]$, we have $e^{|\xi - \eta|} \leq
c_u(a, b)$.  Applying inequality~\eqref{EqnBivSandwich} pointwise and
integrating gives
\begin{align*}
\frac{1}{2 c_u(a, b)} \Ginfo(a, b) & \leq \Dinfo_u(a, b) \leq 2 c_u(a,
b) \Ginfo(a, b).
\end{align*}
Rearranging proves inequality~\eqref{EqnDinfoIncrementEstimate}.  For
the log-odds midpoint, we have $c_{u_\star}(a, b) =
\sqrt{\odds(b)/\odds(a)}$, which gives
inequality~\eqref{EqnDinfoMidpointEstimate}.

It remains to prove the
representation~\eqref{EqnDinfoKLRepresentation} of the one-step KL
error.  From the integral representation in~\Cref{PropBivIntegral}, we
have
\begin{align*}
\DenKL(a, b) & = \int_\alpha^\beta \int_\eta^\beta \Bivq(\eta, \xi)
\, d\xi \, d\eta \; = \; \int_\alpha^\beta \int_\alpha^\xi \Bivq(\eta,
\xi) \, d\eta \, d\xi,
\end{align*}
where the second equality follows by reversing the order of
integration.  For each $\xi \in [\alpha, \beta]$, define $v \defn
\frac{e^\xi}{1 + e^\xi}$ so that $\xi = \logodds(v)$.  Applying the
first identity above with the pair $(a, v)$ and anchor $u = v$ gives
\begin{align*}
\Dinfo_v(a, v) & = \int_\alpha^\xi \Bivq(\eta, \xi) \, d\eta.
\end{align*}
Substituting this identity into the preceding expression yields
\begin{align*}
\DenKL(a, b) & = \int_\alpha^\beta \Dinfo_{\frac{e^\xi}{1 + e^\xi}}
\left(a, \frac{e^\xi}{1 + e^\xi}\right) \, d\xi \; = \; \int_a^b
\frac{\Dinfo_v(a, v)}{v (1 - v)} \, dv,
\end{align*}
where the final equality follows from the change of variables $\xi =
\logodds(v)$, for which $d\xi = dv/{v (1 - v)}$.  This proves the
claim~\eqref{EqnDinfoKLRepresentation}.

\section{Discussion}

In this paper, we have put forth the \igclong (\igc) as an
information-geometric quantity for measuring the structure of
product-reference diffusion (\prd) paths for discrete sampling.  It
provides quantitative bounds on the iteration complexity of various
\prd samplers, both those using constant stepsizes and those with
optimized stepsizes according to a square-root criterion.  Sampling
complexity is not determined solely by the amount of dependence in the
target distribution, but also by where and how rapidly this dependence
is removed by the noising process.  Our \igc theory formalizes this
idea through a second derivative of mutual information along the
diffusion path.  Moreover, we provide a class of statistical
estimators for \igc increments, as well as for the KL error itself,
based on a Bregman divergence that emerges from our theory.  Thus, the
\igc measure provides a unified pathwise description of local
factorization error, finite-step KL discretization, optimal
scheduling, and classical multivariate dependence.

Our perspective makes it possible to compare different mechanisms for
destroying information.  The \igc theory allows the terminal reference
to be an arbitrary product distribution. Each reference induces a
different \igc complexity, thereby yielding a principled way of
studying how the choice of reference affects sampling complexity.  We
gave some preliminary results in this direction
in~\Cref{SecReference}, and we suspect that a more systematic
investigation could be fruitful.  More broadly, our earlier
work~\cite{Wai_UGC} showed that masking diffusion can be characterized
by a related information-geometric functional known as the \ugc
measure.  Together with the \igc representation developed here, these
results provide a principled basis for comparing masking and noising
diffusions on a given target distribution, and for understanding when
one mechanism may be preferable to the other.


\subsubsection*{Acknowledgements}
This work was partially supported by a Guggenheim Fellowship, an NSF
grant (DMS-2311072), and the Ford Professorship at MIT.  We thank
Yuting Wei for her inspiring talk during the MIT Statistics and Data
Science conference in spring 2026.

\bibliographystyle{alpha_initials} {\small{
    \bibliography{discrete_refs} }}


\appendix

\section{Parameterization in terms of the LOO posterior}
\label{AppLOO}
The ordinary and leave-one-out posteriors are related by Bayes' rule:
\begin{subequations}\begin{align}
  \label{EqnDenoiserLooBridge}
  \Den_{i,a}(w, x)
  & = \frac{\Refdist_{i,a}(x_i \mid w)
    \DenLoo_{a,i}(w, x_{-i})}
  {\displaystyle \sum_{u \in \Alphabet}
    \Refdist_{i,a}(x_i \mid u) \DenLoo_{a,i}(u, x_{-i})}.
\end{align}
Substituting~\eqref{EqnDenoiserLooBridge} into
\eqref{EqnDenoiserBridge} gives the equivalent leave-one-out
representation
\begin{align}
  \label{EqnLooBridge}
  \Prob_i^{a \to b}(y \mid x)
  & = \frac{\Refdist_{i,a/b}(x_i \mid y)
    \displaystyle \sum_{w \in \Alphabet}
    \Refdist_{i,b}(y \mid w) \DenLoo_{a,i}(w, x_{-i})}
  {\displaystyle \sum_{w \in \Alphabet}
    \Refdist_{i,a}(x_i \mid w) \DenLoo_{a,i}(w, x_{-i})}.
\end{align}
\end{subequations}

\section{Proof of~\Cref{PropBivIntegral}}
\label{SecProofBivIntegral}

It only remains to prove representations~\eqref{EqnBivTC} of the total
correlation and dual total correlation in terms of the bivariate \igc
density~\eqref{EqnDefnBivQdens}.

Define the functions $\Tfun(\xi) \defn \TC(\Xvar_{\newrevinv(\xi)})$
and $\Dfun(\eta) \defn \DTC(\Xvar_{\newrevinv(\eta)})$.  By the
definition of mutual information, we have $\Ifun_i(\xi, \eta) =
\Ent(\Xvar_{\newrevinv(\xi), i}) + \Ent(\Xvar_{\newrevinv(\eta), -i})
- \Ent\big( \Xvar_{\newrevinv(\xi), i}, \Xvar_{\newrevinv(\eta), -i}
\big)$, and consequently
\begin{align*}
  \partial_1 \Ifun_i(\xi, \eta) = \frac{d}{d\xi}
  \Ent(\Xvar_{\newrevinv(\xi), i}) - \left. \partial_{\lambda_i}
  \Ent(\Xvar_{\rbold(\blam)}) \right|_{\blam = \xi \evec_i + \eta
    (\onevec - \evec_i)}.
\end{align*}
Moreover, note that we have $\frac{d}{d\xi}
\Ent(\Xvar_{\newrevinv(\xi)}) \; = \; \sum_{i = 1}^{\usedim}
\left. \partial_{\lambda_i} \Ent(\Xvar_{\rbold(\blam)}) \right|_{\blam
  = \xi \onevec}$.  Thus, by differentiating the definition of total
correlation, we obtain
\begin{subequations}
  \begin{align}
\label{EqnTCDeriv}    
  \Tfun'(\xi) & = \frac{d}{d\xi} \left\{ \sum_{i = 1}^{\usedim}
  \Ent(\Xvar_{\newrevinv(\xi), i}) - \Ent(\Xvar_{\newrevinv(\xi)})
  \right\} \; = \; \sum_{i = 1}^{\usedim} \partial_1 \Ifun_i(\xi,
  \xi).
\end{align}

Similarly, we have $\partial_2 \Ifun_i(\xi, \eta) = \frac{d}{d\eta}
\Ent(\Xvar_{\newrevinv(\eta), -i}) - \sum_{j \neq i}
\left. \partial_{\lambda_j} \Ent(\Xvar_{\rbold(\blam)}) \right|_{\blam
  = \xi \evec_i + \eta (\onevec - \evec_i)}$, as well as the identity
$\sum_{i = 1}^{\usedim} \sum_{j \neq i} \left. \partial_{\lambda_j}
\Ent(\Xvar_{\rbold(\blam)}) \right|_{\blam = \eta \onevec} = (\usedim
- 1) \frac{d}{d\eta} \Ent(\Xvar_{\newrevinv(\eta)})$.  Thus,
differentiating the formula for dual total correlation yields
\begin{align}
  \label{EqnDTCDeriv}
  \Dfun'(\eta) & = \frac{d}{d\eta} \left\{ \sum_{i = 1}^{\usedim}
  \Ent(\Xvar_{\newrevinv(\eta), -i}) - (\usedim - 1)
  \Ent(\Xvar_{\newrevinv(\eta)}) \right\} \; = \; \sum_{i =
    1}^{\usedim} \partial_2 \Ifun_i(\eta, \eta).
\end{align}
\end{subequations}

Combining these identities with the boundary relations $\Ifun_i(\xi,
-\infty) = \Ifun_i(-\infty, \eta) = 0$ yields the identities
\begin{align*}
  \Tfun'(\xi) & = \sum_{i = 1}^{\usedim} \int_{-\infty}^{\xi}
  \partial_1 \partial_2 \Ifun_i(\xi, \eta) \, d\eta \; = \;
  \int_{-\infty}^{\xi} \Bivq(\eta, \xi) \, d\eta, \qquad \mbox{as well
    as} \\
  \Dfun'(\eta) & = \sum_{i = 1}^{\usedim} \int_{-\infty}^{\eta}
  \partial_1 \partial_2 \Ifun_i(\xi, \eta) \, d\xi \; = \;
  \int_{-\infty}^{\eta} \Bivq(\eta, \xi) \, d\xi.
\end{align*}
Since $\Tfun(-\infty) = \Dfun(-\infty) = 0$, $\Tfun(\infty) =
\TC(\Zvar)$, and $\Dfun(\infty) = \DTC(\Zvar)$, the two
claims~\eqref{EqnBivTC} follow from the fundamental theorem of
calculus.

\end{document}